\documentclass{article}
\PassOptionsToPackage{dvipsnames,table}{xcolor}
\usepackage{preprint,times}
\iclrpreprintcopy

\usepackage{xcolor}
\usepackage{iftex}
\ifXeTeX
  \usepackage{newunicodechar}
  \newunicodechar{’}{\textquoteright}
  \newunicodechar{“}{\textquotedblleft}
  \newunicodechar{”}{\textquotedblright}
\fi
\usepackage{microtype}
\usepackage{graphicx}
\usepackage{subcaption}
\usepackage{booktabs}
\usepackage{tabularx}

\usepackage{colortbl}

\usepackage{placeins}

\newcommand{\captionlabel}[2]{\begingroup
  \setlength{\fboxsep}{1pt}\raisebox{0pt}[0pt][0pt]{\colorbox{#1}{\strut\textcolor{black}{#2}}}\endgroup
}
\newcommand{\greenlabel}{\captionlabel{green!12}{Green}}
\newcommand{\redlabel}{\captionlabel{red!9}{Red}}
\newcommand{\graylabel}{\captionlabel{black!6}{Gray}}

\usepackage[table]{xcolor}

\newcommand{\ablstat}[2]{$#1\pm#2$}
\newcommand{\ablworse}[1]{\cellcolor{red!9}#1}
\newcommand{\ablsimilar}[1]{\cellcolor{black!6}#1}
\newcommand{\ablproposed}[1]{\cellcolor{green!12}#1}
\usepackage{hyperref}

\usepackage[utf8]{inputenc}
\usepackage[T1]{fontenc}
\usepackage{url}
\usepackage{amsfonts}
\usepackage{nicefrac}

\usepackage[most]{tcolorbox}

\usepackage{wrapfig}
\usepackage{multirow}
\usepackage{makecell}
\usepackage{colortbl}

\definecolor{lightgray}{gray}{0.8}

\newcommand{\framework}{\textsc{PolicyLift}\xspace}

\usepackage{amsmath,amssymb,mathtools}
\usepackage{etoolbox}
\newif\ifcompactmaindisplays
\compactmaindisplaystrue
\newcommand{\tightendisplayspacing}{\ifcompactmaindisplays
    \ifdim\abovedisplayskip>2pt
      \addtolength{\abovedisplayskip}{-2pt}\fi
    \ifdim\belowdisplayskip>2pt
      \addtolength{\belowdisplayskip}{-2pt}\fi
    \ifdim\abovedisplayshortskip>2pt
      \addtolength{\abovedisplayshortskip}{-2pt}\fi
    \ifdim\belowdisplayshortskip>2pt
      \addtolength{\belowdisplayshortskip}{-2pt}\fi
  \fi
}
\AtBeginEnvironment{equation}{\tightendisplayspacing}
\AtBeginEnvironment{align}{\tightendisplayspacing}
\usepackage{bbm}
\usepackage[ruled,vlined,linesnumbered]{algorithm2e}
\usepackage{amsthm}
\RestyleAlgo{ruled}
\theoremstyle{plain}
\newtheorem{theorem}{Theorem}[section]
\newtheorem{proposition}[theorem]{Proposition}
\newtheorem{lemma}[theorem]{Lemma}
\newtheorem{corollary}[theorem]{Corollary}
\theoremstyle{definition}

\theoremstyle{remark}

\newcommand{\argmax}{\mathop{\mathrm{arg\,max}}}

\newcommand{\1}{\mathbf{1}}

\usepackage{xspace}

\newcommand*\diff{\mathop{}\!\mathrm{d}}

\newcommand*\circledgreen[1]{\tikz[baseline=(char.base)]{
  \node[shape=circle, draw=ForestGreen!60, fill=ForestGreen!10, thick, inner sep=1pt] (char) {\scriptsize\textsf{#1}};
}}
\newcommand*\circledred[1]{\tikz[baseline=(char.base)]{
  \node[shape=circle, draw=BrickRed!60, fill=BrickRed!10, thick, inner sep=1pt] (char) {\scriptsize\textsf{#1}};
}}
\newcommand*\circledblue[1]{\tikz[baseline=(char.base)]{
  \node[shape=circle, draw=NavyBlue!60, fill=NavyBlue!10, thick, inner sep=1pt, font=\upshape] (char) {\scriptsize\textsf{#1}};
}}
\usepackage{pifont}
\newcommand{\greencheck}{\textcolor{ForestGreen}{\checkmark}}
\newcommand{\redcross}{\textcolor{BrickRed}{\ding{55}}}

\makeatletter
\def\maketag@@@#1{\hbox{\m@th\normalfont\normalsize#1}}
\makeatother

\newtcolorbox{graybox}{
    colback=gray!10,
    colframe=gray!40,
    boxrule=0.3pt,
    arc=2mm,
    left=1mm,
    right=1mm,
    top=0.5mm,
    bottom=0.5mm,
    after skip=2pt
}

\usepackage{enumitem}

\usepackage{xr-hyper}

\usepackage{titlesec}
\titlespacing*{\section}{0pt}{1.5ex plus .3ex minus .2ex}{1ex plus .2ex}
\titlespacing*{\subsection}{0pt}{1.3ex plus .3ex minus .2ex}{0.5ex plus .1ex}

\usepackage[textsize=tiny]{todonotes}

\usepackage{caption}
\usepackage{adjustbox}

\usepackage{float}

\title{Generative Policy Improvement \\ from Observational Data}

\author{\makebox[0.96\textwidth][c]{Jonas Schweisthal$^{1,2}$ \quad
Dennis Frauen$^{1,2}$ \quad
Stefan Feuerriegel$^{1,2}$ \quad
Emma Brunskill$^{3}$\thanks{EB is currently also a researcher at Google DeepMind. EB conducted this research solely in their capacity at Stanford. Google DeepMind had no involvement in the research or preparation of this manuscript.}}\\[0.5em]
$^{1}$LMU Munich \quad
$^{2}$Munich Center for Machine Learning (MCML)\\
$^{3}$Department of Computer Science, Stanford University\\[0.3em]
Correspondence to: \texttt{jonas.schweisthal@lmu.de}
}

\begin{document}

\maketitle

\begin{abstract}
We study how to improve decision policies from observational data to \textit{generate} better actions when actions are complex, structured objects, such as text messages or images. This task is subject to two major challenges: (i) action-level overlap is often too weak for reliable off-policy evaluation and improvement, and (ii) generating new structured actions often requires expensive training or fine-tuning (e.g., of an LLM) and can be ill-defined due to information loss in action representations.
To address both, we propose \framework, a three-stage framework for \textit{generative policy improvement} with structured actions from observational data. In Stage 1, \framework learns a small set of discrete action codes that capture high-level strategies in the observed action space and are tailored to support policy improvement under limited overlap. In Stage 2, it learns a context-dependent policy over these action codes. At deployment time, Stage 3 ``lifts'' the learned code-level policy back to the original action space by generating actions from the historical behavior system in a code-conditioned way.
We show theoretically that improving the policy over action codes guarantees improvement of the induced policy in the original action space. We demonstrate \framework across experiments in synthetic and real-world text generation tasks with consistent policy improvements over the behavior policy.
\end{abstract}

\section{Introduction}

Learning decision policies from observational data, i.e., off-policy learning (OPL), asks how to improve a behavior policy using logged interactions rather than new online experiments~\citep{ athey2021policy, Dudik.2014, Kallus.2022, Swaminathan.2015}. Existing work studies off-policy evaluation (OPE) and OPL primarily under relatively simple action spaces, such as binary, categorical, or low-dimensional interventions, where actions correspond to choices like whether to treat a patient, which product to recommend, or which price to offer~\citep{kitagawa2018should, athey2021policy, Kallus.2022}. In many modern systems, actions are \emph{complex} and \emph{high-dimensional}; for example, an online platform can generate a personalized headline or advertisement~\citep{matias2021upworthy,meguellati2024good} or a mental health support system may suggest an empathic reply in an online peer-support conversation~\citep{sharma2023human}.

We study OPL for such complex action spaces. Given logged triples $(X,A,Y)$ from a behavior policy, where $X$ is a context, $A$ is an action, and $Y$ is an observed outcome, the goal is to learn a target policy that can \emph{generate} improved structured actions at deployment time. In online mental health support, $X$ may encode the help-seeker's post and conversation history, $A$ is a generated or suggested supportive response, and $Y$ measures empathy, engagement, or conversation quality. In headline optimization, $X$ may encode an article excerpt, $A$ is a generated headline, and $Y$ is reader engagement. 

However, such complex, high-dimensional actions are often implicitly generated by a latent lower-dimensional structure, such as how style and content might be used to generate a specific supportive exercise suggestion.
In these settings, policy improvement requires not selecting over a small discrete set of actions, but \emph{generating} high-dimensional, complex actions like such messages. Two \textbf{challenges} make this substantially harder than standard off-policy learning:

\begin{graybox}

\circledred{1}~\textit{\textbf{Limited overlap.}} OPE and OPL require overlap ~\citep{Dudik.2014, Kallus.2022}, which means that the behavior policy covers the actions a new target policy would select. For complex actions, this is often implausible: the behavior system concentrates on a vanishingly small subset of the action space for each context. Even when the behavior system is stochastic, 
the randomness is often local 
and thus does not induce meaningful coverage across qualitatively different actions. 
Existing approaches to limited overlap either constrain the policy class~\citep{kallus.2021, schweisthal2023reliable, zhao2024positivity} or optimize conservative objectives~\citep{jin2025policy}, but these operate in the original action space and cannot resolve the fundamental coverage gap in high-dimensional structured spaces. 
\vspace{0.25em}

\circledred{2}~\textit{\textbf{Information loss.}} It is unclear a~priori how to learn action representations that both make overlap plausible and still allow mapping coarse latent actions back to raw actions. A representation-level decision (e.g., a text style) is not yet a deployable action: deployment still requires a specific response in the original action space (e.g., natural language). Prior work learns improved policies in a subspace of the actions~\citep{saito2022off,peng2023offline}, but if the representation does not fully mediate the effect of the action on the outcome~\citep{saito2023off}, instantiating raw actions (e.g., via a decoder) can induce a systematic mismatch between the value optimized in representation space and the value delivered in the original action space. We formalize this later as a \emph{deployment gap}.
\end{graybox}
\vspace{0.3em}

To address these challenges, we propose \framework, a framework for \emph{generative} policy improvement from observational data. 
\framework first maps each observed action to a discrete \emph{action code}, where \emph{codes} represent high-level action strategies that are easier to compare, estimate, and optimize under limited overlap. 
Next, it learns a context-dependent policy over these codes using standard off-policy learning tools. 
Finally, it lifts the learned code-level policy back to the original action space by using the behavior system itself as a generator. 
We call this \emph{behavior-consistent lifting} (BCL) and show that, when the behavior system is sampled until an action with the target code is found, BCL eliminates the deployment gap, so improvements in true code-level value transfer exactly to the deployed policy over structured actions.

Our contributions are the following. 
\textbf{(i)} We formulate OPL with \textit{generative} structured
actions, where the target policy must generate new actions rather than select from a fixed action set, and propose \framework to address this setting. 
\textbf{(ii)} We derive various theoretical guarantees for our framework. In particular, we prove that improving the policy over action codes guarantees improvement of the induced policy in the original action space.
\textbf{(iii)} We empirically validate our theory and show that \framework improves over baselines in LLM-based text generation tasks.

\section{Related Work}
\label{sec:related_work}

\setlength{\columnsep}{1.6em}
\setlength{\intextsep}{0.2em}

\begin{wrapfigure}{r}{0.50\textwidth}
\vspace{-0.8\baselineskip}
\centering

\captionsetup{type=table,font=footnotesize,labelfont=bf,skip=2pt}

\caption{\scriptsize
\emph{Legend.} \greencheck = addressed; \redcross = not addressed;
(\greencheck) = partially addressed; (\redcross) = assumed or not enforced; -- = not applicable.
}
\label{tab:related_work}

\vspace{-0.4em}

{\tiny
\setlength{\tabcolsep}{1.4pt}
\renewcommand{\arraystretch}{0.82}

\begin{adjustbox}{max width=\linewidth}
\begin{tabular}{>{\raggedright\arraybackslash}p{0.65\linewidth}ccccc}
\toprule
& \makecell{\textbf{Unseen structured}\\ \textbf{actions}}
& \makecell{\textbf{Learned}\\ \textbf{repr.}}
& \makecell{\textbf{Overlap}\\ \textbf{controlled}}
& \makecell{\textbf{Learns}\\ \textbf{policy}}
& \makecell{\textbf{Generative}\\ \textbf{deployment}} \\
\midrule
Overlap-constrained OPL~\citep{kallus.2021,schweisthal2023reliable,zhao2024positivity,jin2025policy}
& \redcross & -- & \greencheck & \greencheck & -- \\
Action embeddings / MIPS~\citep{saito2022off, saito2023off}
& \redcross & (\greencheck) & (\redcross) & \redcross & \redcross \\
PoTEC~\citep{saitopotec2025}
& \redcross & (\greencheck) & (\redcross) & \greencheck & \redcross \\
Learned action embeddings~\citep{peng2023offline, cief2024learning}
& \redcross & \greencheck & (\redcross) & \redcross & \redcross \\
Prompt-guided generation~\citep{kiyohara2025off}
& (\greencheck) & \redcross & (\redcross) & \greencheck & \greencheck \\
\midrule
\textbf{\framework (ours)}
& \greencheck & \greencheck & \greencheck & \greencheck & \greencheck \\
\bottomrule
\end{tabular}
\end{adjustbox}
}

\vspace{0.2em}

\vspace{-0.3\baselineskip}
\end{wrapfigure}

\textbf{Policy learning from observational data.}
OPL studies how to improve a behavior policy from logged interactions without collecting new online data. 
Standard approaches estimate or optimize policy value using inverse propensity weighting, doubly robust estimation, or regression adjustment~\citep{Swaminathan.2015, Dudik.2014, Kallus.2022, athey2021policy}. 
A separate stream addresses limited overlap by modifying the policy-learning problem itself; for example, by restricting, retargeting, or otherwise regularizing the learned policy toward better-supported interventions~\citep{kallus.2021, schweisthal2023reliable, zhao2024positivity}, or using pessimistic value criteria that penalize policies with high uncertainty under the logged data~\citep{jin2025policy,Swaminathan.2015}. 
\textit{However, these approaches operate in simple or low-dimensional action spaces and do not address policy learning with structured actions.}

\setlength{\columnsep}{1.6em}
\setlength{\intextsep}{0.2em}

\begin{wrapfigure}[12]{R}{0.60\textwidth}
\vspace{-0.4\baselineskip}
\centering

\includegraphics[width=\linewidth]{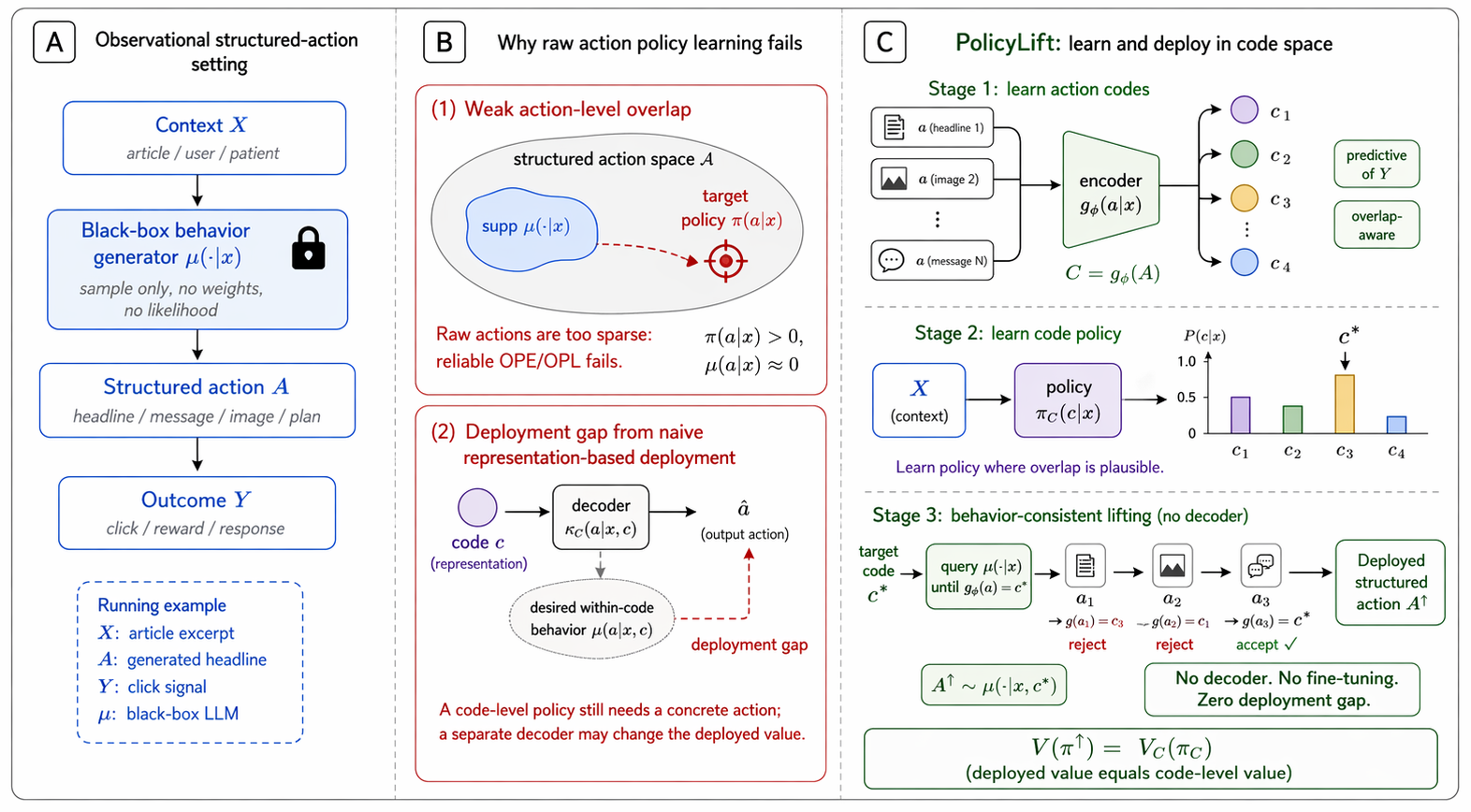}

\vspace{-0.5em}

\caption{\textbf{Overview.} Generative policy improvement from observational data with complex, structured actions.}
\label{fig:framework}

\vspace{-0.4\baselineskip}
\end{wrapfigure}

\textbf{OPE/OPL with complex action spaces} (an extended related work is in Appendix~\ref{app:related_work}). 
Several works extend OPE and OPL beyond binary or categorical actions. For continuous actions, semiparametric and doubly robust methods learn policies without discretizing the action space~\citep{Kallus.2018e,Demirer.2019}. For OPE in large action spaces, action-embedding methods replace action-level propensities with marginalized propensities (MIPS) over embeddings to reduce variance \citep{saito2022off}, potentially also accounting for residual variance within clusters (OffCEM,~\citep{saito2023off}). PoTEC extends this idea to OPL by decomposing the policy into cluster selection and within-cluster action selection \citep{saitopotec2025}, using pre-defined clusters or a separate method to a~priori perform action clustering (e.g.~\citealp{peng2023offline}). \emph{These methods operate on a fixed, enumerable action set, and do not extend to the generation of new structured actions.} Other approaches learn action groups or action embeddings from logged data~\citep{peng2023offline,cief2024learning}. \emph{However, these methods target OPE, not OPL, and assume sufficient overlap in the learned representation. In contrast, \framework learns action codes that encourage overlap, optimizes a policy over those codes, and generates code-conditional structured actions by sampling from the behavior system.} 

Closest to our setting is \citet{kiyohara2025off}, which studies OPL for sentence generation, but where the logged action is an observed discrete prompt and deployment uses a frozen generator (behavior system) to map prompts to text. \emph{In contrast, \framework learns a code policy from logged structured actions \textbf{without} requiring logged prompts, and deploys it using only sampling access to the behavior system.}

\section{Problem Setup}
\label{sec:setup}

\textbf{Off-policy learning.}
We consider the standard setup for policy learning from observational data~\citep{ Dudik.2014, athey2021policy, Kallus.2022}.
We observe i.i.d.\ data $\mathcal D=\{(X_i,A_i,Y_i)\}_{i=1}^n$ generated by a historical \emph{behavior policy} $\mu(\cdot\mid x)$, where $X\in\mathcal X$ is a context, $A\in\mathcal A$ is a structured action such as an LLM-generated text, and $Y\in\mathbb R$ is an outcome.
Importantly, we do \textit{not} assume a fixed pre-defined set of finite actions. We let $Y(a)$ denote the potential outcome under action $a$~\citep{Rubin.1974}. For any policy $\pi(\cdot\mid x)$ over $\mathcal A$, the \emph{policy value} is $V(\pi) := \mathbb E[Y(\tilde A)], \; \tilde A\sim\pi(\cdot\mid X).$ 
In addition to logged data, we assume sampling access to the behavior policy at deployment time: given a context $x$, we can draw candidate structured actions from $\mu(\cdot\mid x)$, but we do not require access to its weights, likelihoods, or training data. This captures settings where the historical behavior system is available as a black-box generator.
Our goal is to learn a target policy that can \emph{generate complex raw actions} at deployment.

\textbf{Identification.}
We adopt the standard identification assumptions for OPL~\citep{Rubin.1974, athey2021policy}, namely, \emph{consistency} (i.e., $Y=Y(A)$) and \emph{ignorability} (i.e., $Y(a)\perp A\mid X$ for all $a\in\mathcal A$). Under these assumptions, the conditional expected potential outcome is identified by the regression function $m(x,a):= \mathbb E[Y(a)\mid X=x] = \mathbb E[Y\mid X=x,A=a]$, and the value of any target policy satisfies $V(\pi)=\mathbb E[m(X,\tilde A)]$ with $\tilde A\sim\pi(\cdot\mid X)$. 
In addition, identification of $V(\pi)$ for a target policy $\pi$ requires  \emph{overlap} (also called positivity or common support)~\citep{rosenbaum1983central, imbens2015causal, damour2021overlap}, namely, the behavior policy covers the support of $\pi$; i.e., for $P_X$-almost all $x\in\mathcal X$,
$\mu(a\mid x)>0$ whenever $\pi(a\mid x)>0$.
Overlap violations bias value estimates~\citep{felicioni2022off, sachdeva2020off}, and weak overlap inflates the variance of importance-weighted estimators~\citep{saito2022off, sachdeva2024off, crump2009dealing} and the extrapolation error of regression-based ones~\citep{farajtabar2018more, voloshin2021empirical}.

\section{PolicyLift Framework}
\label{sec:codes}

\textbf{Overview:} We propose \framework, a three-stage \textit{generative} framework for \textit{policy improvement} from observational
data over structured action spaces. Figure~\ref{fig:framework} shows the framework, which separates policy learning from action generation. Briefly, in \textbf{Stage}~\circledblue{1}, we encode each observed action into a discrete action code $C=g_\phi(A)\in\{1,\ldots,K\}$ via a tailored learned encoder. This maps the high-dimensional structured action space $\mathcal A$ to a low-dimensional code space where overlap is more plausible. In \textbf{Stage}~\circledblue{2}, we learn a context-dependent code-level policy $\pi_C(c\mid x)$ from the induced data $(X_i,C_i,Y_i)$. In \textbf{Stage}~\circledblue{3}, we deploy $\pi_C$ through our so-called \emph{behavior-consistent lifting}: for a target code sampled from $\pi_C(\cdot\mid x)$, we draw candidate actions from the behavior system $\mu(\cdot\mid x)$ and return an action whose code matches the target. This makes the learned policy \textit{generative}: \framework outputs a deployed policy sampler $\pi^\uparrow(\cdot\mid x)$ over original structured actions. Algorithm~\ref{alg:policylift} shows the pseudocode.

\textbf{Stage~\circledblue{1}: Learning action codes.} Stage 1 learns an encoder $g_\phi: \mathcal A \to \{1,\dots,K\}$ that maps each action to a discrete \emph{action code} $C = g_\phi(A)$, with induced code-level behavior policy $\mu_C(c\mid x) := \Pr(C = c \mid X = x)$. Encoder learning targets two properties that determine the value of the code-level policy learning in Stage~2, namely the code should be (i)~\emph{predictive} of the outcome, and (ii)~the code-conditional context distribution should \emph{overlap} with the marginal. We describe both below:

$\bullet$\,\textbf{(i) \textit{Predictiveness.}} The code $C$ should be predictive of $Y$ given $X$, i.e., the code regression function $m_C(x, c) := \mathbb E[Y \mid X=x, C=c]$ should vary meaningfully across $c$. If $C$ carries no information about $Y$ beyond $X$, no code-level policy can improve over $\mu$. Thus, the underlying variation determines how much a code-level policy can improve over the behavior policy.

$\bullet$\,\textbf{(ii) \textit{Overlap.}} The code-level overlap assumption needed for policy learning and deployment is formulated wrt. the behavior code propensity $\mu_C(c\mid x)$. Stage~1 uses a distributional surrogate for this assumption instead: the code-conditional context distribution $Q_c^\phi := P(X \mid C = c)$ should be close to the marginal $P_X$ for each code $c$. These two notions are related by Bayes' rule. Let $p_C^\phi(c) := \Pr(C=c)$. Whenever \(Q_c^\phi \ll P_X\) (i.e., \(Q_c^\phi\) is absolutely continuous with respect to \(P_X\), so it assigns no mass outside the support of \(P_X\)), we aim for
$\mu_C(c\mid x)=p_C^\phi(c)\frac{\diff Q_c^\phi}{\diff P_X}(x)$
for $P_X$-almost all $x$. Thus, in the ideal balanced case $Q_c^\phi=P_X$, the code propensity is context-independent, $\mu_C(c\mid x)=p_C^\phi(c)$. More generally, distributional balance controls how much code propensities vary with $x$, while marginal code balance controls whether any code is globally rare. This balance reduces extrapolation when code-level outcome regressors are evaluated under the
marginal context distribution.

\textbf{Objective.} We learn the action encoding $g_\phi$ by minimizing a trade-off between the two properties,
\begin{align}
\min_{\phi,\theta}\quad
&\underbrace{\frac{1}{n}\sum_{i=1}^n \big(Y_i - f_\theta(X_i, g_\phi(A_i))\big)^2}_{\text{predictiveness }}
+
\beta\underbrace{\sum_{c=1}^K \widehat D\big(\widehat P_X, \widehat Q_c^{\phi}\big)}_{\text{distributional overlap }}
+
\gamma\underbrace{\mathrm{KL}\big(\widehat p_C^\phi \,\|\, \mathrm{Unif}_K\big)}_{\text{marginal code balance}}.
\label{eq:stage1_obj}
\end{align}
Here $f_\theta$ is an outcome predictor, $\widehat D$ is an empirical distributional distance such as sliced Wasserstein, and $\widehat p_C^\phi$ is the empirical marginal distribution of learned codes. The trade-off weight $\beta \geq 0$ balances attainable value against finite-sample reliability, while $\gamma\geq 0$ prevents globally rare or empty codes (and a trivial single-code minimzer of the overlap term). 

\textbf{Stage~\circledblue{2}: Learning a code-level policy.} Given the encoder from Stage~1, we re-encode the logged actions as $C_i = g_\phi(A_i)$ and learn a code-level policy $\pi_C(\cdot\mid x)$ on $(X_i, C_i, Y_i)$.

$\bullet$\,\emph{Feasible policy class.} For a threshold $\tau > 0$, we restrict to
\begin{equation}
\Pi_{C,\tau} := \{\pi_C : \pi_C(c\mid x) > 0 \Rightarrow \mu_C(c\mid x) \geq \tau\}.
\label{eq:Pi_C_tau}
\end{equation}
This constraint is the pointwise code-level overlap condition used for policy learning and deployment. 
The Stage 1 distributional and marginal-balance regularizers make the constraint easier to satisfy by encouraging code propensities to be less context-dependent and not globally rare. However, these regularizers are only surrogates: distributional balance alone does not imply a uniform lower bound on $\mu_C(c\mid x)$.

$\bullet$\,\emph{Objective.} We learn $\pi_C \in \Pi_{C,\tau}$ by maximizing a regression-adjusted (RA) value estimator
\begin{equation}
\widehat V_{\mathrm{RA}}(\pi_C) := \frac{1}{n}\sum_{i=1}^n \sum_{c=1}^K \pi_C(c\mid X_i)\,\widehat m(X_i, c),
\label{eq:V_RA}
\end{equation}
where $\widehat m(x,c)$ is a learned regressor for $m_C(x,c)$. We use RA rather than IPW or DR estimators because Stage 1 directly regularizes the context regions on which the code-level regressors are evaluated (Theorem~\eqref{thm:ra}).

\begin{wrapfigure}[17]{r}{0.45\textwidth}
\vspace{-0.3cm}
\centering

\begin{minipage}{\linewidth}
\begin{algorithm}[H]
\scriptsize
\DontPrintSemicolon
\caption{\framework}
\label{alg:policylift}
\KwIn{logged data $\mathcal D=\{(X_i,A_i,Y_i)\}_{i=1}^n$; number of codes $K$; overlap weights $\beta\ge 0, \gamma\ge 0$; threshold $\tau>0$; draw budget $T_{\max}$.}
\KwOut{deployed policy sampler $\widehat\pi^\uparrow(\cdot\mid x)$.}

\BlankLine
\textbf{Stage}~\circledblue{1}: \textbf{Learn action codes}\;
Fit encoder $\widehat g_\phi$ and outcome predictor $f_\theta$ 
\;
Re-encode logged actions: $C_i \leftarrow \widehat g_\phi(A_i)$\;

\BlankLine
\textbf{Stage}~\circledblue{2}: \textbf{Learn code-level policy}\;
Fit code-level regression $\widehat m(x,c) \approx m_C(x,c)$ on $\{(X_i,C_i,Y_i)\}_{i=1}^n$\;
Solve $\widehat\pi_C \in \argmax_{\pi_C\in\Pi_{C,\tau}} \widehat V_{\mathrm{RA}}(\pi_C)$ 
\;

\BlankLine
\textbf{Stage}~\circledblue{3}: \textbf{Deployment at context $x$}\;
Sample target code $c^\star \sim \widehat\pi_C(\cdot\mid x)$\;
\For{$t=1,\ldots,T_{\max}$}{
    Sample $a_t\sim\mu(\cdot\mid x)$\;
    \If{$\widehat g_\phi(a_t)=c^\star$}{
        \Return $a_t$\;
    }
}
\Return an independent draw from $\mu(\cdot\mid x)$\;
\end{algorithm}
\end{minipage}
\vspace{-0.1cm}
\end{wrapfigure}

\textbf{Stage~\circledblue{3}: Behavior-consistent lifting.} The learned code-level policy $\pi_C$ is not itself a deployable policy yet: at deployment time, the system must output a concrete action $a \in \mathcal A$. We propose a simple but effective solution as follows.
For a context $x$, Stage~3 first selects a target code given the learned code policy $c^\star \sim \pi_C(\cdot\mid x)$, then generates an action whose code is $c^\star$ by \emph{behavior-consistent lifting}: we sample $a \sim \mu(\cdot\mid x)$ and accept the first draw with $g_\phi(a) = c^\star$.
In practice, we cap the number of draws at $T_{\max}$ and fall back to $\mu(\cdot\mid x)$ if no match is found. The threshold $\tau$ also controls deployment cost: for supported target codes, the expected number of draws scales as $1/\mu_C(c^\star\mid x)$, and the feasible class keeps this quantity bounded.

\section{Theoretical Analysis}
\label{sec:theory}

We now analyze three key properties of \framework: (1)~\emph{behavior-consistent lifting} makes a code-level policy deployable in the original action space without changing its value; (2)~\emph{predictiveness of action codes} determines the attainable improvement of code-level policy learning; and (3)~\emph{regression-adjusted policy learning} with balanced action codes to reduce extrapolation error. Together, these results justify the three design choices of \framework: behavior-consistent lifting, predictive code learning, and overlap regularization. All proofs are in Appendix~\ref{app:proofs}.

\subsection{Behavior-consistent lifting}
\label{sec:theory_lifting}

We first analyze \textbf{Stage}~\circledblue{3}.
The key quantity is the potential value difference between the code-level policy and the deployed action-level policy. We define an \emph{instantiation rule} (i.e., decoder or generator) as a conditional distribution $\kappa(a\mid x,c)$ over actions given a context-code pair. Combined with a code policy $\pi_C$, it induces a deployed policy on actions~\citep{saitopotec2025} $\pi^{\kappa}(a\mid x) := \sum_{c=1}^K \pi_C(c\mid x)\,\kappa(a\mid x,c)$.

Behavior-consistent lifting corresponds to $\kappa(\cdot\mid x,c) = \mu(\cdot\mid x,c)$, the within-code behavior distribution. We denote the resulting deployed policy by
\begin{equation}
\pi^{\uparrow}(a\mid x) := \sum_{c=1}^K \pi_C(c\mid x)\,\mu(a\mid x,c).
\label{eq:pi_up_def}
\end{equation}
The acceptance loop used by \framework implements this instantiation rule exactly in the ideal uncapped case: Proposition~\ref{prop:reject_lift} in Appendix~\ref{app:acceptance} shows that accepting the first draw from $\mu(\cdot\mid x)$ whose code equals $c$ returns a sample from $\mu(\cdot\mid x,c)$. Thus, behavior-consistent lifting is \emph{not} an additional learned decoder, but \emph{a sampler of the within-code behavior distribution}.

\begin{theorem}[Value under code-based deployment]
\label{thm:deployment_gap}
Under consistency and ignorability, for any code policy $\pi_C$ and instantiation rule $\kappa$, the policy value at deployment is given by
\begingroup
\begin{equation}
\begin{aligned}
V(\pi_\kappa)
&=
\underbrace{V_C(\pi_C)}_{\text{code-level value}}
+
\underbrace{
\mathbb E\left[
\sum_{c=1}^K \pi_C(c\mid X)\Delta_\kappa(X,c)
\right]
}_{\text{deployment gap}},
\end{aligned}
\label{eq:deployment_gap}
\end{equation}
\endgroup
where
$V_C(\pi_C) := \mathbb E\!\left[\sum_c \pi_C(c\mid X)\,m_C(X,c)\right]$,
$\Delta_\kappa(x,c) := \mathbb E_{A\sim \kappa(\cdot\mid x,c)}[m(x,A)] - \mathbb E_{A\sim \mu(\cdot\mid x,c)}[m(x,A)]$. 
In particular, $\Delta_\kappa(x,c) = 0$ if $\kappa(\cdot\mid x,c) = \mu(\cdot\mid x,c)$, or if $m(x,a)$ is constant over actions receiving mass within code $c$.\footnote{This condition is analogous to Assumption 3.2 of \citet{saito2022off}: the action embedding fully mediates the effect of the action on the outcome.}
\end{theorem}

Intuitively, Theorem~\ref{thm:deployment_gap} decomposes deployed value into (i)~the code-level policy value $V_C(\pi_C)$ and (ii)~the  \textit{deployment gap}, which measures how the instantiation rule $\kappa$ changes the within-code action distribution relative to the behavior system. 
The deployment gap can be nonzero whenever $\kappa$ reweights actions within a code in a way that changes their expected outcomes.\footnote{The deployment gap can be positive if the instantiation rule improves over the code-level policy value; PoTEC~\citep{saitopotec2025} uses a two-stage OPL approach that exploits this for a finite, enumerable action set. However, in complex structured action spaces such as text, new actions must be generated for a new context $x$. The key challenge is that $\kappa$ can then place mass outside the within-code support of $\mu$, making the deployed value \textbf{not} identified from logged data when this occurs with positive probability under the learned policy (Proposition~\ref{prop:nonident_support_mismatch}). The deployed policy can then be even \emph{worse} than the original behavior policy.} 
Our behavior-consistent lifting avoids this issue by setting $\kappa(\cdot\mid x,c)=\mu(\cdot\mid x,c)$ directly. Consequently, Theorem~\ref{thm:deployment_gap} gives $\Delta_\kappa\equiv 0$, so the deployment gap vanishes: the value optimized at the code level matches exactly the value achieved after deployment, as shown in the following corollary:
\begin{corollary}[Value equivalence]
\label{cor:value_equivalence}
Under consistency and ignorability, for any $\pi_C$ satisfying code-level overlap, the behavior-consistent deployed policy from Eq.~\eqref{eq:pi_up_def} satisfies $V(\pi^{\uparrow}) = V_C(\pi_C).$
\end{corollary}

Corollary~\ref{cor:value_equivalence} makes code-level policy learning a valid deployment target: a policy can be optimized and evaluated in the simpler code space, and \emph{the same improvement value is realized when the policy is lifted back to the original structured action space}. In particular, any code-level improvement transfers directly to a deployed improvement over the behavior policy in the original action space.

\begin{corollary}[Improvement transfer]
\label{cor:improvement_transfer}
Under the conditions of Corollary~\ref{cor:value_equivalence}: if $V_C(\pi_C) \geq V_C(\mu_C)$, then $V(\pi^{\uparrow}) \geq V(\mu)$.
\end{corollary}

Behavior-consistent lifting also has bounded expected sampling costs. For a target code $c^\star$, the number of behavior-system samples required for acceptance is controlled by the code propensity $\mu_C(c^\star\mid x)$. Therefore, the overlap threshold $\tau$ ensures code-level support while also bounding the expected number of behavior-system samples required in expectation, with $T_{\max}$ providing an additional safeguard on worst-case cost.

\begin{lemma}[Sampling cost]
\label{lem:acceptance_rate}
For a target code $c^\star$ at context $x$, the expected number of draws from $\mu(\cdot\mid x)$ until $g_\phi(a) = c^\star$ is $1/\mu_C(c^\star\mid x)$. For $\pi_C \in \Pi_{C,\tau}$, this is bounded by $1/\tau$.
\end{lemma}

\subsection{Predictiveness and the attainable code-level value}
\label{sec:theory_predictiveness}

We next
lower-bound the attainable code-level policy value in terms of the 
predictive power of $C$. 
Let $\bar m(x) := \mathbb E[Y\mid X=x]$ denote the context-only regression. The baseline value of the behavior policy is $V(\mu) = \mathbb E[\bar m(X)]$. The best attainable code-level value under encoder $g_\phi$ is the oracle value
\begingroup
\addtolength{\abovedisplayskip}{-2pt}
\begin{align}
V_C^{\star}(\phi)
&:= \sup_{\pi_C}\ \mathbb E\!\left[\sum_{c=1}^K \pi_C(c\mid X)\,m_C(X,c)\right]
= \mathbb E\!\left[\max_{c:\,\mu_C(c\mid X)>0} m_C(X,c)\right],
\label{eq:V_oracle}
\end{align}
\endgroup
where the supremum is over code policies assigning mass only to codes with $\mu_C(c\mid x)>0$, and the second equality follows by linearity in $\pi_C(\cdot\mid x)$. We define the predictiveness signal of the encoder as the conditional variance of the code regression, i.e., i.e.,
$\mathrm{Sig}(\phi) := \mathbb E\!\left[\mathrm{Var}\!\left(m_C(X,C)\mid X\right)\right]$.
If $C$ is conditionally constant given $X$, then $m_C(X,C) = \bar m(X)$ almost surely and $\mathrm{Sig}(\phi) = 0$; in that case, we have $V_C^\star(\phi) = V(\mu)$, and no code-level policy can improve over the behavior policy.

\begin{theorem}[Predictiveness implies value potential]
\label{thm:sig}
Assume $|Y|\leq B$ a.s. for some $B>0$. Then, $\mathrm{Sig}(\phi)$ admits the equivalent characterizations
\begin{align}
\mathrm{Sig}(\phi)
&= \mathbb E\!\left[(m_C(X,C)-\bar m(X))^2\right]
= \mathbb E\!\left[(Y-\bar m(X))^2\right]-\mathbb E\!\left[(Y-m_C(X,C))^2\right],
\label{eq:sig_mse}
\end{align}
and the oracle value in Eq.~\eqref{eq:V_oracle} satisfies the lower bound
\begin{equation}
V_C^{\star}(\phi)\ \ge\ V(\mu)\ +\ \frac{1}{2B}\,\mathrm{Sig}(\phi).
\label{eq:oracle_lb}
\end{equation}
\end{theorem}

The second equality in Eq.~\eqref{eq:sig_mse} identifies $\mathrm{Sig}(\phi)$ as the reduction in optimal prediction error when predicting $Y$ from $(X,C)$ instead of from $X$ alone. Learning $g_\phi$ to make the induced code $C=g_\phi(A)$ predictive of $Y$, i.e., to reduce the MSE of predicting $Y$ from $(X,C)$, therefore increases $\mathrm{Sig}(\phi)$. The inequality in Eq.~\eqref{eq:oracle_lb} then translates this signal into attainable improvement over the behavior policy. $\Rightarrow$ \emph{This directly relates to the predictiveness term in the \textbf{Stage}~\circledblue{1} loss.}

Theorem~\ref{thm:sig} has a different focus from existing OPE analyses for action embeddings or clusters~\citep{saito2022off,saito2023off,peng2023offline}. Existing analyses characterize estimator bias, variance, or MSE under embedding-mediation, local-correctness, or within-group outcome-similarity conditions; here, we lower-bound the value attainable by code-level policy learning in terms of an MSE-reduction signal. 

\subsection{Overlap and RA estimation error}
\label{sec:theory_overlap}

Recall that code policy learning proceeds by first estimating the outcome regressor $\hat m(\cdot,c)$ and then producing a code-level policy by maximizing (or soft-maximizing) this function for a given input context. 
The central challenge for regression is extrapolation. The code regressor $\widehat m(\cdot, c)$ is fit on contexts drawn from $Q_c^\phi$, but the RA value estimator from Eq.~\eqref{eq:V_RA} integrates $\widehat m$ against the full marginal $P_X$ weighted by $\pi_C(\cdot\mid x)$. If $Q_c^\phi$ is far from $P_X$ for some code, learning $\widehat m(\cdot, c)$ requires extrapolating to context regions with few or no logged observations under code $c$. Let $e_c(x) := \widehat m(x, c) - m_C(x, c)$ denote the prediction error. To quantify how this extrapolation inflates the estimation error of $\widehat V_{\mathrm{RA}}$, we posit the following shift bound, which relates prediction error under $P_X$ to prediction error under $Q_c^\phi$:
\begin{equation}
\mathbb E_{P_X}[e_c(X)^2]\ \leq\ \mathbb E_{Q_c^{\phi}}[e_c(X)^2]\ +\ \lambda_{\mathrm{ov}}\,D(P_X, Q_c^{\phi}),
\label{eq:overlap_shift_assump}
\end{equation}
for some distance $D$ on distributions over $\mathcal X$ and a constant $\lambda_{\mathrm{ov}} > 0$.\footnote{Standard choices include Wasserstein-1 (under Lipschitzness of $e_c^2$) or kernel MMD (under RKHS control of $e_c^2$) \citep{johansson2022generalization}.} Recall $V_C(\pi_C)$ is the true value of $\pi_C$.
We can upper bound  
the RA estimation error into the prediction error on the code-conditional distribution, by an overlap penalty measuring the total distance between $P_X$ and the code-conditional distributions $Q_c^\phi$, and a standard sample-complexity term:
\begin{theorem}[RA estimation error under limited overlap]
\label{thm:ra}
Assume $|Y| \leq B$ a.s. and $|\widehat m(x, c)| \leq B$ for all $(x, c)$. Let $\widehat m$ be independent of the evaluation sample (e.g., via cross-fitting), let $\Pi_C$ be a finite code-policy class, and assume Eq.~\eqref{eq:overlap_shift_assump} holds. Then, with probability at least $1-\delta$, simultaneously for all $\pi_C \in \Pi_C$, it holds that
\begingroup
\fontsize{9.5}{11}\selectfont
\begin{equation}
\big|V_C(\pi_C) - \widehat V_{\mathrm{RA}}(\pi_C)\big|
\ \leq\  \sqrt{\sum_{c=1}^K \mathbb E_{Q_c^{\phi}}\!\big[e_c(X)^2\big]\ +\ \lambda_{\mathrm{ov}}\sum_{c=1}^K D(P_X, Q_c^{\phi})}
 +\ B\sqrt{\frac{2\log(2|\Pi_C|/\delta)}{n}}.
\label{eq:ra_bound_main}
\end{equation}
\endgroup
\end{theorem}
Note the overlap penalty is propensity-free: $D(P_X, Q_c^\phi)$ depends only on the empirical distributions of $X$ and $X\mid C=c$, so it can be minimized directly during representation learning. This justifies the second term of the \textbf{Stage}~\circledblue{1} loss. Theorem~\ref{thm:ra} is unlike prior work on OPE under limited overlap, which focuses on variance inflation of importance-weighted estimators~\citep{crump2009dealing, Dudik.2014, farajtabar2018more, Swaminathan.2015, voloshin2021empirical}; in contrast, our theorem characterizes RA extrapolation error and adapts the balanced-representation bound~\citep{johansson2016learning, shalit2017estimating} to learned discrete codes and code-level policy value.

\section{Neural implementation}
\label{sec:implementation}

We describe a neural implementation of \framework.
We provide a full pseudocode with implementation choices in Appendix~\ref{app:implementation} in Alg.~\ref{alg:policylift_impl_long}.

\textbf{Stage}~\circledblue{1}\textbf{: Learning action codes.} For learning the encoder $g_\phi$, we use a two-branch network. The \emph{action branch} maps $A$ to logits over $\{1,\dots,K\}$ and applies a Gumbel--Softmax layer~\citep{jang2017categorical} to obtain a differentiable code vector $\tilde C$ during training; at inference, we take the arg-max to obtain discrete codes. The \emph{context branch} produces an embedding $\psi(X)$, which we concatenate with $\tilde C$ and feed into an outcome predictor $f_\theta$ trained with MSE. For the overlap distance $\widehat D$, we use the sliced Wasserstein distance~\citep{kolouri2019generalized} on minibatches of $\psi(X)$, comparing the pooled distribution to each within-code distribution $\psi(X) \mid \tilde C = c$; we stop gradients through $\psi$ in this term, so that the overlap penalty adjusts code assignments via $\phi$ rather than by distorting the context representation. Because the overlap term is minimized trivially by collapsing all actions into a single code, we add a marginal code-balance penalty $\gamma\,\mathrm{KL}(\widehat p_C \,\|\, \mathrm{Unif})$, where $\widehat p_C$ is the empirical code histogram in the minibatch. The total Stage~1 loss is thus $\mathcal L_{\text{pred}} + \beta\mathcal L_{\text{ov}} + \gamma\mathcal L_{\text{cb}}$.

\textbf{Stage}~\circledblue{2}\textbf{: Learning a code-level policy.} To fit the code-level regression $\widehat m(x, c)$, we initialize the predictor $f_\theta$ from Stage~1 and fine-tune the prediction head on $(X_i, C_i, Y_i)$ with the encoder frozen, reusing the Stage~1 representation while adapting the head to hard codes. We then define the code-level policy by $\pi_C(x) \in \arg\max_{c:\,\widehat\mu_C(c\mid x)\ge \tau} \widehat m(x,c)$, where $\widehat\mu_C(c\mid x)$ is a propensity proxy fit by cross-entropy on $(X_i,C_i)$. Using the estimated propensity, this approximates the feasible-class constraint $\Pi_{C,\tau}$ while optimizing the regression-adjusted objective in Eq.~\eqref{eq:V_RA}.

\textbf{Stage}~\circledblue{3}\textbf{: Deployment.} For deployment at context $x$, we select the target code greedily as $c^\star = \pi_C(x)$. We then run the acceptance loop of Algorithm~\ref{alg:policylift} with budget $T_{\max}$. Since the policy optimization already restricts to supported codes satisfying $\widehat\mu_C(c^\star\mid x)\ge \tau$, no additional overlap fallback is required at deployment. The cap $T_{\max}$ still controls worst-case deployment cost and limits deviation from the behavior policy when matching target codes is difficult.

\section{Experiments}
\label{sec:experiments}

We evaluate \framework in three steps. \circledgreen{1}~We use synthetic experiments to test the theoretical properties of behavior-consistent lifting, outcome-predictive code learning, and overlap-aware code learning. \circledgreen{2}~To show the modularity of \framework, we evaluate the policy improvement across three datasets while varying the Stage~\circledblue{1} code and Stage~\circledblue{2} policy learning objective. \circledgreen{3}~We assess practical applicability on a challenging, realistic, public language-generation benchmark and compare against prompt-based methods with additional information.

\textbf{Datasets.}
For evaluation, we use three datasets that progress from isolated mechanisms to fully realistic generated text. (i) Our \emph{synthetic} DGP uses numeric structured actions and separately controls within-code reward variation, code predictiveness, and overlap (Appendix~\ref{app:synthetic_diagnostics}). (ii) We construct a \emph{headline generation} benchmark that combines real article contexts with GPT-5.4 mini-generated headline actions and a controlled reward function. We use MiniLM embeddings with constructed simulator features (Appendices~\ref{app:semisynth_construction} and~\ref{app:comparison_implementation}) to ensure that a learnable signal can be retrieved from the noisy data. (iii) We test on the public \emph{OfflinePrompts}~\citep{kiyohara2025off} benchmark that provides a MovieLens-based personalization benchmark with Mistral-generated text and a public reward simulator trained on real-world user responses. Here, \framework learns from context, generated text, and reward and does \textbf{not} require observing the generating prompt like other methods (Appendix~\ref{app:offlineprompts}).

\par
\begin{wrapfigure}{r}{0.75\textwidth}
\vspace{-8pt}
  \centering
  \captionsetup{skip=2pt}
  \includegraphics[width=\linewidth]
    {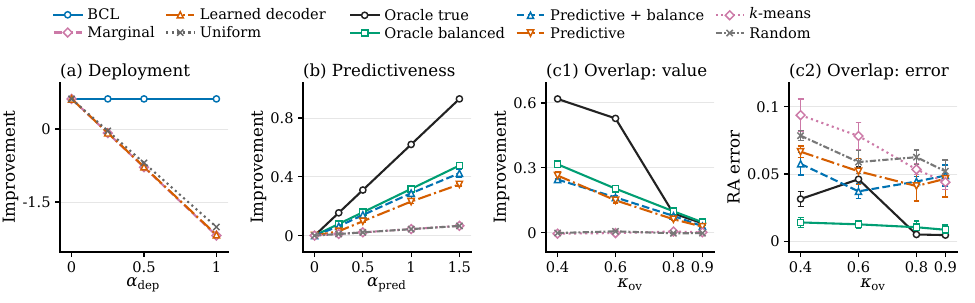}
  \caption{\textbf{Synthetic experiments.}
  (a) Deployment; (b) predictiveness; (c1--c2) overlap.
 (a,b,c1) Policy improvement (higher better);
  (c2) RA estimation error (lower better).
5 seeds, std errors  where available.
  Each diagnostic fixes other stages;
  see Appendix~\ref{app:synthetic_protocols}.}
  \label{fig:synthetic_diagnostics}
\end{wrapfigure}
\noindent
\textbf{Baseline.} Related methods \citep{saito2022off, saito2023off, saitopotec2025} are \emph{not directly applicable} due to important differences; these works  (a) have mostly used a large but finite action space, which prevents generating new actions (e.g., the within-cluster policy of PoTEC \citep{saitopotec2025}); or (b) make additional assumptions such as access to an observed discrete action cluster (like discrete prompt strategies \citep{kiyohara2025off}) or are OPE methods rather than OPL \citep{saito2022off, saito2023off}. However, to demonstrate the modularity of our framework, we show we can integrate the aforementioned methods into \framework for extending them to \emph{generative policy learning}.
We proceed as follows: For Stage~1, we replace the codes learned with our full objective with random codes, $k$-means codes, or predictive-only codes ($\beta=\gamma=0$), while keeping RA and behavior-consistent lifting. For Stage~2, we retain the proposed codes but use code-IPS, code-DR, or kernel-IS objectives to directly adapt the methods from OPE MIPS \citep{saito2022off}, OPE OffCEM \citep{saito2023off}, and sentence-level importance weighting \citep{kiyohara2025off}. On OfflinePrompts (where \textit{additionally} prompt strategies are observed), we also compare with prompt-based regression and DSO~\citep{kiyohara2025off}. These baseline can use prompt information that is \textit{unavailable} to \framework. Appendix~\ref{app:comparison_implementation} describes the implementations. Our primary metric is policy improvement over behavior, $V(\hat\pi^\uparrow)-V(\mu)$. Tables report mean $\pm$ standard error over five seeds.

\circledgreen{1}~\textbf{Validation of theoretical properties.}
Figure~\ref{fig:synthetic_diagnostics} isolates the following three mechanisms (see Appendix~\ref{app:synthetic_protocols} for details): \textbf{(i)} To compare BCL against alternative decoding (uniform, context-free, and learned decoding), we vary within-code reward effects with oracle codes fixed. \textbf{(ii)} To isolate the effect of the code quality, we vary the predictive signal with oracle Stage~2 optimization. \textbf{(iii)} We vary logging-policy concentration to examine overlap and estimation error. \textbf{Findings:} \textbf{(i)} BCL preserves code-policy value and guarantees improvement while alternative decoders show the deployment gap, i.e., even \emph{value deterioration} (Figure~\ref{fig:synthetic_diagnostics}, (a)). \textbf{(ii)} Outcome-predictive codes increase the value, i.e., enabling larger policy improvement (Figure~\ref{fig:synthetic_diagnostics}, (b)). The overlap diagnostic demonstrates the tradeoff between representation quality (predictiveness) and reliable policy learning (stage 2 error) (Figure~\ref{fig:synthetic_diagnostics} (c)), showing the code value in theory for improvement (c1), and how consistently we can learn a good policy on these codes from finite data (c2).  $\Rightarrow$~\emph{The experiments demonstrate the importance of our three main modeling choices: predictive codes, support-constrained policy learning, and behavior-consistent lifting.}

\par
\begin{wraptable}{r}{0.70\textwidth}
\vspace{-6pt}
\centering
\captionsetup{position=top,skip=2pt}
\caption{\textbf{\framework ablations.}
Policy improvement ($\uparrow$); cells report mean $\pm$ SE
over 5 seeds. \greenlabel{}: proposed configuration.
\redlabel{}: mean $\pm$ SE intervals entirely below proposed.
\graylabel{}: overlapping intervals.}
\label{tab:policylift_modularity}

\scriptsize
\setlength{\tabcolsep}{2pt}
\renewcommand{\arraystretch}{1.05}

\begin{tabular}{@{}p{0.36\linewidth}ccc@{}}
\toprule
Variant & Synthetic & Headlines & \shortstack{Offline Prompts} \\
\midrule

\ablproposed{\textbf{  \shortstack{Proposed: \\ Pred.\&bal.\ codes + Code-RA}}}
& \ablproposed{\ablstat{0.1579}{0.0083}}
& \ablproposed{\ablstat{0.0482}{0.0095}}
& \ablproposed{\ablstat{0.1332}{0.0034}} \\

\midrule
\multicolumn{4}{@{}l}{\emph{Stage 1: code replacement}} \\

Random
& \ablworse{\ablstat{0.0049}{0.0040}}
& \ablsimilar{\ablstat{0.0400}{0.0084}}
& \ablworse{\ablstat{0.0116}{0.0019}} \\

$k$-means
& \ablworse{\ablstat{0.0015}{0.0040}}
& \ablworse{\ablstat{0.0219}{0.0000}}
& \ablworse{\ablstat{0.0902}{0.0036}} \\

Predictive only
& \ablsimilar{\ablstat{0.1547}{0.0076}}
& \ablworse{\ablstat{0.0059}{0.0026}}
& \ablsimilar{\ablstat{0.1498}{0.0280}} \\

\midrule
\multicolumn{4}{@{}l}{\emph{Stage 2: objective replacement}} \\

Code-IPS \citep{saito2022off}
& \ablsimilar{\ablstat{0.1437}{0.0115}}
& \ablsimilar{\ablstat{0.0522}{0.0058}}
& \ablworse{\ablstat{0.1031}{0.0053}} \\

Code-DR \citep{saito2023off}
& \ablsimilar{\ablstat{0.1334}{0.0257}}
& \ablsimilar{\ablstat{0.0529}{0.0063}}
& \ablworse{\ablstat{0.1088}{0.0030}} \\

Kernel-IS \citep{kiyohara2025off}
& \ablworse{\ablstat{0.1083}{0.0229}}
& \ablsimilar{\ablstat{0.0448}{0.0095}}
& \ablworse{\ablstat{0.1112}{0.0086}} \\

\bottomrule
\end{tabular}
\end{wraptable}
\noindent
\circledgreen{2}~\textbf{Modularity and policy improvement.} 
Table~\ref{tab:policylift_modularity} shows results for various configurations of \framework. Importantly, we show positive mean improvement for \emph{every} configuration of \framework on all three datasets. Hence, \framework is flexible and can be combined with various stage 1 and 2 alternatives to extend them for generative policy improvement. Between the configurations, we observe the following: \circledblue{1}~\textbf{Code learning:} Our proposed code learning improves over the behavior policy and achieves higher mean improvement than methods without predictiveness ($k$-means, random) across all datasets. Additional balancing improves on the synthetic and semi-synthetic datasets with challenging overlap, with an increase in improvement on the headline data from $0.0059$ to $0.0482$, and stays competitive for OfflinePrompts. \circledblue{2} \textbf{Policy learning:} With our proposed codes fixed, our proposed RA objective achieves the highest improvement on synthetic data and OfflinePrompts, and remains competitive with the alternative objectives on headlines. $\Rightarrow$ \emph{Our proposed configuration is a very strong default, and our \framework also can extend   alternative estimators to  produce effective generative policies.}

\par
\begin{wraptable}{R}{0.36\textwidth}
\centering
\captionsetup{position=top,skip=2pt}
\caption{\textbf{OfflinePrompts.}
Policy value ($\uparrow$), mean $\pm$ SE over five seeds.
\greenlabel{}: \framework. \redlabel{}: behavior.
\graylabel{}: references using additional prompt information
unavailable to \framework; their mean $\pm$ SE intervals
overlap its interval.}
\label{tab:offlineprompts_references}

\scriptsize
\setlength{\tabcolsep}{3pt}
\renewcommand{\arraystretch}{1.05}

\begin{tabular}{@{}lc@{}}
\toprule
Method & Policy value \\
\midrule

\rowcolor{red!9}
Behavior policy
& $0.1031\pm0.0000$ \\

\rowcolor{green!12}
\textbf{\framework}
& $0.2363\pm0.0034$ \\

\midrule
\multicolumn{2}{@{}l}{\emph{Information-advantaged references}} \\

\rowcolor{black!6}
Prompt-based regression
& $0.2266\pm0.0109$ \\

\rowcolor{black!6}
DSO-style prompt policy
& $0.2684\pm0.0458$ \\

\bottomrule
\end{tabular}
\end{wraptable}
\noindent
\circledgreen{3}~\textbf{Policy improvement without logged prompts.}
OfflinePrompts tests whether the framework remains effective on fully realistic generated text even beyond our carefully designed headline benchmark. Table~\ref{tab:offlineprompts_references} compares the same proposed configuration against methods that evaluate and select different prompts that generated the text explicitly; hence, \framework uses \emph{strictly less information} than the baselines. \framework increases the policy value from $0.1031$ to $0.2363$, a relative improvement of approximately $129\%$. Its mean exceeds prompt-based regression, and  compares similarly to DSO; however, DSO suffers from  higher variation across seeds.
$\Rightarrow$~\emph{\framework demonstrates clear policy improvement and strong performance also in challenging realistic settings and even despite access to less information.}

\textbf{Discussion.}
We prove theoretically and show empirically that \framework improves policies with structured actions by learning in a supported code space and deploying through behavior-consistent lifting, which removes the deployment gap.
Our framework has two main limitations: it requires sampling access to the behavior system at deployment time, and preserves the within-code behavior distribution but does not optimize among actions with the same code. Future work could study less conservative lifting rules, further improved code representations, and online deployed evaluations of rolled-out policies beyond public benchmarking, such as A/B testing in personalized online marketing. To avoid negative impacts, we emphasize that deployment must be done in a safe, ethical, and reliable manner.

\compactmaindisplaysfalse
\clearpage
\subsection*{Funding statement}
This work has been supported by the German Federal Ministry of Education and Research (Grant: 01IS24082).

\bibliography{bibliography}

\newpage
\appendix

\newpage
\section{Extended discussion and limitations}
\label{sec:discussion}
\framework addresses policy improvement with structured actions by learning a code-level policy and deploying it through behavior-consistent lifting. This separates policy learning from action generation: the policy is optimized in a supported code space, while deployment remains in the original action space by sampling from the behavior system. We show theoretically and empirically that this removes the deployment gap for behavior-consistent lifting and transfers code-level improvement to deployed improvement over the behavior policy.

The approach has different limitations. First, \framework requires sampling access to the behavior system at deployment time. It does not apply when the logging policy cannot be queried or when deployment must use a different generator. Second, behavior-consistent lifting preserves the within-code behavior distribution. This gives an identified deployment rule, but it does not optimize among actions with the same code. If reliable within-code outcome scores or model-weight access are available, less conservative instantiation rules may improve further, but they require additional assumptions to avoid a deployment gap. Third, the method depends on learning codes that are both predictive of outcomes and sufficiently supported. If no code representation satisfies this predictiveness--overlap trade-off (e.g. due to a close-to-deterministic behavior system), the attainable improvement is limited. All three represent interesting directions for future work, as does 
evaluating our framework in real-world settings, e.g., via A/B testing in personalized online marketing. In such applications, our framework could improve personalization and decision-making in generative systems, though care should be taken to ensure that optimization objectives remain aligned with user well-being and broader platform goals.

\newpage

\section{Extended Related Work}\label{app:related_work}
\textbf{LLMs and generative alignment.}
Several adjacent literatures improve policies that can produce or select complex outputs, but rely on different information or deployment assumptions.
RLHF and direct preference optimization fine-tune generative models from preference data and require access to model weights~\citep{ouyang2022training, rafailov2023direct}.
Controlled decoding steers a frozen generator at inference time, but requires decoding-time access and an external reward, value, or attribute scorer~\citep{mudgal2024controlled, yang2021fudge}.
Inference-time reranking and Best-of-$N$ sampling select among generated candidates using a learned reward model, typically trained from preference data~\citep{stiennon2020learning, nakano2021webgpt, gao2023scaling}.
Plain Best-of-$N$ can overoptimize such proxy rewards as the candidate set grows~\citep{gao2023scaling}; recent regularized variants mitigate this effect, but still optimize in the original structured action space using reward-model scores~\citep{gui2024bonbon, huang2025best, jinnai2025regularized}.
In contrast, \framework starts only from logged observational triples $(X,A,Y)$: it does not assume preference data, reward-model supervision, model-weight access, action-level likelihoods, or a scorer that can evaluate arbitrary generated candidates.

\textbf{Offline reinforcement learning}.
Offline RL learns from logged data using support constraints, pessimism, or behavior regularization~\citep{fujimoto2019off, kumar2020conservative, kostrikov2022offline}, but typically operates directly in the original action space rather than learning a deployable representation-level policy. In a similar setting, deep generative policy-learning methods use generative models and latent action decoders for policies ~\citep{chen2024deep, zhou2021plas}, but require training such models and do not rule out a deployment gap when latent actions are decoded into original actions.

\textbf{Balanced representation learning in causal inference}. Related work on balanced representations for treatment effect estimation uses distributional distance regularization on covariate representations to adjust for covariate shift between treatment groups~\citep{johansson2016learning, shalit2017estimating, johansson2022generalization}, but assumes discrete treatments. \citet{harada2021graphite} extend this balancing to high-dimensional treatments. However, these works consider only treatment effect estimation and not policy learning. Relatedly, \citet{shi2025leveraging} learn representations of high-dimensional treatments for causal effect estimation and adaptive online experimentation, while our focus is offline policy learning from fixed observational data.

\begin{table}[h]
\centering
\footnotesize
\caption{\textbf{Differences to selected related streams.}
We compare \framework to the most relevant lines of work along the two challenges emphasized in this paper: overlap in structured action spaces and deployment of a representation-level policy as a policy over original structured actions.}
\label{tab:related_work_extended}
\begin{tabularx}{\linewidth}{@{}>{\raggedright\arraybackslash}Xccccc@{}}
\toprule
& \makecell{\textbf{Obs.}\\ \textbf{outcomes}}
& \makecell{\textbf{Structured}\\ \textbf{actions}}
& \makecell{\textbf{Learned}\\ \textbf{repr.}}
& \makecell{\textbf{Overlap}\\ \textbf{controlled}}
& \makecell{\textbf{Deploys in}\\ \textbf{orig. space}} \\
\midrule
Standard OPE/OPL~\citep{Dudik.2014, Swaminathan.2015, kallus.2021}
& \greencheck & \redcross & -- & (\redcross) & -- \\
Action embeddings / clusters~\citep{saito2022off, saito2023off, saitopotec2025}
& \greencheck & \greencheck & (\greencheck) & (\redcross) & \redcross \\
Learned action representations~\citep{peng2023offline, cief2024learning}
& \greencheck & \greencheck & \greencheck & (\redcross) & \redcross \\
Prompt-guided generation~\citep{kiyohara2025off}
& \greencheck & \greencheck & \redcross & (\redcross) & (\greencheck) \\
Decoder-based offline RL~\citep{fujimoto2019off, zhou2021plas}
& \greencheck & (\greencheck) & \greencheck & (\greencheck) & (\greencheck) \\
Inference-time reranking / Best-of-$N$~\citep{gao2023scaling, gui2024bonbon, huang2025best}
& \redcross & \greencheck & \redcross & -- & \greencheck \\
\midrule
\textbf{\framework (ours)}
& \greencheck & \greencheck & \greencheck & \greencheck & \greencheck \\
\bottomrule
\end{tabularx}

\vspace{0.25em}
\begin{minipage}{0.93\linewidth}
\footnotesize
\emph{Legend.} \greencheck = addressed; \redcross = not addressed; 
(\greencheck) = partially addressed; (\redcross) = assumed or not enforced; -- = not applicable. 
``Overlap controlled'' means that overlap is encouraged or enforced by the learning procedure rather than only assumed for identification. 
``Deploys in original space'' means that the learned policy is deployable as a policy over original structured actions.
\end{minipage}
\end{table}

\textbf{Research gap.} 
Existing methods address only parts of our setting: overlap-constrained OPL targets simpler action spaces, representation-based OPE/OPL typically assumes fixed actions and representation-level overlap, and generative alignment or offline RL requires stronger access such as preference data, reward models, model weights, or learned generators.
This leaves open generative policy improvement from observational data.
\framework addresses this gap by learning overlap-aware action codes, optimizing a code-level policy, and lifting the policy through the behavior system.

\newpage

\section{Proofs and Extensions}
\label{app:proofs}

\subsection{A technical assumption on within-code conditionals}
\label{app:conditional}

For the results below, we assume existence of regular conditional distributions for the within-code behavior.

\paragraph{Assumption A3 (Within-code behavior conditional).}
For each $x \in \mathcal X$ and $c \in \{1,\dots,K\}$ with $\mu_C(c \mid x) > 0$, there exists a conditional distribution $\mu(\cdot \mid x, c)$ such that
\begin{equation}
\mathbb P(A \in S \mid X = x, C = c) = \int_S \mu(a \mid x, c)\,\mathrm da
\end{equation}
for all measurable $S \subseteq \mathcal A$ (interpreting $\mathrm da$ generically; in discrete $\mathcal A$ replace by sums).

\subsection{Deployment gap and context-free decoding}

\label{app:proof_decoding}

\begin{proof}[Proof of Theorem~\ref{thm:deployment_gap}]
Let $\tilde A \sim \pi^\kappa(\cdot \mid X)$. By consistency and ignorability,
\begin{align}
V(\pi^\kappa)
&= \mathbb E_X\!\left[\mathbb E[Y(\tilde A) \mid X]\right]
= \mathbb E_X\!\left[\sum_{a \in \mathcal A}\pi^\kappa(a \mid X)\,m(X, a)\right] \\
&= \mathbb E_X\!\left[\sum_{c=1}^K \pi_C(c \mid X)\sum_{a \in \mathcal A}\kappa(a \mid X, c)\,m(X, a)\right] \\
&= \mathbb E_X\!\left[\sum_{c=1}^K \pi_C(c \mid X)\,\mathbb E_{A \sim \kappa(\cdot \mid X, c)}[m(X, A)]\right].
\end{align}
Subtracting and adding $m_C(X, c) = \mathbb E_{A \sim \mu(\cdot \mid X, c)}[m(X, A)]$, which holds by Lemma~\ref{lem:code_mixture} below, yields
\begin{align}
V(\pi^\kappa)
&= \mathbb E_X\!\left[\sum_{c=1}^K \pi_C(c \mid X)\,m_C(X, c)\right]
+ \mathbb E_X\!\left[\sum_{c=1}^K \pi_C(c \mid X)\,\Delta_\kappa(X, c)\right] \\
&= V_C(\pi_C) + \mathbb E_X\!\left[\sum_{c=1}^K \pi_C(c \mid X)\,\Delta_\kappa(X, c)\right],
\end{align}
which is~\eqref{eq:deployment_gap}. The two claims on when $\Delta_\kappa(x, c) = 0$ are immediate from the definition of $\Delta_\kappa$.\end{proof}

\begin{corollary}[Context-free decoding can yield a nonzero deployment gap under perfect marginal fit]
\label{cor:naive_gap_marginal_fit}
Suppose the decoder is context-free, $\kappa(\cdot \mid x, c) \equiv \kappa_0(\cdot \mid c)$, and matches the marginal within-code behavior distribution, $\kappa_0(\cdot \mid c) = \mu(\cdot \mid c) := \mathbb P(A \in \cdot \mid C = c)$. Then
\begin{equation}
V(\pi^{\kappa_0}) - V_C(\pi_C)
= \mathbb E_X\!\left[\sum_{c=1}^K \pi_C(c \mid X)\int_{\mathcal A} m(X, a)\,\big(\mu(\mathrm da \mid c) - \mu(\mathrm da \mid X, c)\big)\right].
\label{eq:naive_gap_marginal_fit}
\end{equation}
Consequently, perfect marginal fit does not guarantee a zero deployment gap.
\end{corollary}

\begin{proof}
Apply Theorem~\ref{thm:deployment_gap} with $\kappa(\cdot \mid x, c) = \mu(\cdot \mid c)$. The deployment gap at $(x, c)$ simplifies to
\begin{align}
\Delta_{\kappa_0}(x, c)
&= \mathbb E_{A \sim \mu(\cdot \mid c)}[m(x, A)] - \mathbb E_{A \sim \mu(\cdot \mid x, c)}[m(x, A)] \\
&= \int_{\mathcal A} m(x, a)\,\big(\mu(\mathrm da \mid c) - \mu(\mathrm da \mid x, c)\big),
\end{align}
which gives~\eqref{eq:naive_gap_marginal_fit} after substitution into~\eqref{eq:deployment_gap}.
\end{proof}

Corollary~\ref{cor:naive_gap_marginal_fit} makes precise why context-free decoding can break alignment between code-level evaluation and deployed value: even perfect reconstruction of the \emph{marginal} within-code distribution averages across contexts, whereas deployment requires the \emph{context-specific} mixture $\mu(\cdot \mid x, c)$.

\begin{proposition}[Non-identifiability under support mismatch]
\label{prop:nonident_support_mismatch}
Fix a context-code pair $(x, c)$ with $\mu_C(c \mid x) > 0$. Suppose there exists an action $a^\star \in \mathcal A$ such that $g_\phi(a^\star) = c$ and
\begin{equation}
\kappa(a^\star \mid x, c) > 0 \qquad \text{but} \qquad \mu(a^\star \mid x, c) = 0.
\label{eq:support_mismatch}
\end{equation}
Then $\mathbb E_{A \sim \kappa(\cdot \mid x, c)}[m(x, A)]$ is not identified from the observational distribution under consistency and ignorability. The policy value $V(\pi^\kappa)$ is likewise not identified if such support mismatch occurs on a set of context-code pairs with positive probability under $P_X(\mathrm dx)\pi_C(c\mid x)$.
\end{proposition}

\begin{proof}
Construct two data-generating processes that share the same behavior policy $\mu$ and agree on all potential outcomes $\{Y(a) : a \neq a^\star\}$, but assign different values to $\mathbb E[Y(a^\star) \mid X = x]$. Because $\mu(a^\star \mid x, c) = 0$, the action $a^\star$ is never observed under the behavior distribution at $(x, c)$, so both processes induce the same observational distribution of $(X, A, Y)$ and the same $m_C(x, c)$. However, since $\kappa(a^\star \mid x, c) > 0$, the quantity $\mathbb E_{A \sim \kappa(\cdot \mid x, c)}[m(x, A)]$ differs between the two processes. Thus, if such support mismatches occur on a set of positive $P_X(\mathrm dx)\pi_C(c\mid x)$-measure, $V(\pi^\kappa)$ is not identified either.
\end{proof}

\subsection{Value equivalence and improvement transfer}
\label{app:proof_equivalence}

We begin with a lemma showing that the code regression function admits a causal interpretation as a within-code mixture of potential-outcome regressions. This is the key identity used in the proof of Theorem~\ref{thm:deployment_gap} and the corollaries that follow.

\begin{lemma}[Code regression as within-code mixture]
\label{lem:code_mixture}
The code regression function $m_C(x, c) := \mathbb E[Y \mid X = x, C = c]$ is defined as a statistical object under the observational distribution of $(X, C, Y)$. Under consistency and ignorability, and under Assumption~A3, it admits the causal interpretation
\begin{equation}
m_C(x, c) = \mathbb E_{A \sim \mu(\cdot \mid x, c)}\!\left[m(x, A)\right]
\label{eq:code_mixture}
\end{equation}
for any $(x, c)$ with $\mu_C(c \mid x) > 0$.
\end{lemma}

\begin{proof}
By definition and consistency,
\begin{align}
m_C(x, c)
= \mathbb E[Y \mid X = x, C = c]
= \mathbb E[Y(A) \mid X = x, C = c].
\end{align}
Applying the tower property by conditioning on $A$ and then using ignorability,
\begin{align}
\mathbb E[Y(A) \mid X = x, C = c]
&= \mathbb E_{A \mid X = x, C = c}\!\left[\mathbb E[Y(A) \mid X = x, A, C = c]\right] \\
&= \mathbb E_{A \mid X = x, C = c}\!\left[\mathbb E[Y(A) \mid X = x, A]\right] \\
&= \mathbb E_{A \mid X = x, C = c}[m(x, A)]
= \mathbb E_{A \sim \mu(\cdot \mid x, c)}[m(x, A)],
\end{align}
where the last equality uses Assumption~A3.
\end{proof}

\begin{proof}[Proof of Corollary~\ref{cor:value_equivalence}]
Apply Theorem~\ref{thm:deployment_gap} with $\kappa(\cdot \mid x, c) = \mu(\cdot \mid x, c)$. Then $\Delta_\kappa(x, c) = 0$ for all $(x, c)$ in the support of $\pi_C$, and~\eqref{eq:deployment_gap} reduces to $V(\pi^\uparrow) = V_C(\pi_C)$.
\end{proof}

\begin{proof}[Proof of Corollary~\ref{cor:improvement_transfer}]
Applying Corollary~\ref{cor:value_equivalence} with $\pi_C = \mu_C$ shows $V(\mu) = V_C(\mu_C)$, since $\mu(a \mid x) = \sum_{c=1}^K \mu_C(c \mid x)\,\mu(a \mid x, c)$. Combining with $V_C(\pi_C) \geq V_C(\mu_C)$ and Corollary~\ref{cor:value_equivalence} yields $V(\pi^\uparrow) = V_C(\pi_C) \geq V_C(\mu_C) = V(\mu)$.
\end{proof}

\subsection{Acceptance sampling}
\label{app:acceptance}
\begin{proposition}[Acceptance sampling implements the behavioral decoder]
\label{prop:reject_lift}
Fix a context $x$ and a target code $c^\star$ with $\mu_C(c^\star \mid x)>0$. Define the behavioral decoder
\begin{equation}
\kappa_\mu(S \mid x,c^\star)
:=
\mathbb P_\mu(A \in S \mid X=x, g_\phi(A)=c^\star)
=
\frac{
\int_S \mathbf 1\{g_\phi(a)=c^\star\}\,\mu(\mathrm da \mid x)
}{
\mu_C(c^\star \mid x)
}
\label{eq:behavioral_decoder}
\end{equation}
for all measurable $S\subseteq\mathcal A$. Let
$A^{(1)},A^{(2)},\ldots \overset{\mathrm{iid}}{\sim}\mu(\cdot \mid x)$ and define
\begin{equation}
T := \inf\{t\geq 1 : g_\phi(A^{(t)})=c^\star\},
\qquad
A^{\mathrm{RS}} := A^{(T)} .
\end{equation}
Then, for all measurable $S\subseteq\mathcal A$,
\begin{equation}
\mathbb P(A^{\mathrm{RS}}\in S \mid X=x; c^\star)
=
\kappa_\mu(S \mid x,c^\star).
\end{equation}
\end{proposition}

\begin{proof}
Let $p_c(x):=\mu_C(c^\star\mid x)$ and
\begin{equation}
q_S(x,c^\star)
:=
\mathbb P_\mu(A\in S, g_\phi(A)=c^\star \mid X=x)
=
\int_S \mathbf 1\{g_\phi(a)=c^\star\}\,\mu(\mathrm da \mid x).
\end{equation}
Since $p_c(x)>0$, the stopping time $T$ is finite almost surely. By independence of the proposal draws,
\begin{align}
\mathbb P(A^{\mathrm{RS}}\in S \mid X=x; c^\star)
&=
\sum_{t=1}^{\infty}
\mathbb P(T=t, A^{(t)}\in S \mid X=x; c^\star) \\
&=
\sum_{t=1}^{\infty}
(1-p_c(x))^{t-1} q_S(x,c^\star) \\
&=
\frac{q_S(x,c^\star)}{p_c(x)} \\
&=
\frac{
\int_S \mathbf 1\{g_\phi(a)=c^\star\}\,\mu(\mathrm da \mid x)
}{
\mu_C(c^\star \mid x)
}
=
\kappa_\mu(S \mid x,c^\star),
\end{align}
where the last equality follows from~\eqref{eq:behavioral_decoder}.
\end{proof}

Fix a context $x$ and a target code $c^\star$ with $\mu_C(c^\star \mid x) > 0$. The ideal acceptance sampler draws $A^{(1)}, A^{(2)}, \ldots \overset{\text{iid}}{\sim} \mu(\cdot \mid x)$ and returns $A^{(T)}$, where $T := \min\{t \geq 1 : g_\phi(A^{(t)}) = c^\star\}$ is the first accepted index.

\begin{proof}[Proof of Lemma~\ref{lem:acceptance_rate}]
For any measurable $S \subseteq \mathcal A$,
\begin{align}
\mathbb P(A \in S, g_\phi(A) = c^\star \mid X = x)
&= \int_S \mu(a \mid x)\,\mathbf 1\{g_\phi(a) = c^\star\}\,\mathrm da, \\
\mathbb P(g_\phi(A) = c^\star \mid X = x)
&= \mu_C(c^\star \mid x).
\end{align}
Conditioning on acceptance yields $\mathbb P(A \in S \mid X = x, g_\phi(A) = c^\star) = \mu(S \mid x, c^\star)$, so the returned action has distribution $\mu(\cdot \mid x, c^\star)$. The acceptance probability per draw is $\mu_C(c^\star \mid x)$, and $T$ is geometric with mean $1/\mu_C(c^\star \mid x)$. On $\Pi_{C,\tau}$, the bound $\mu_C(c^\star \mid x) \geq \tau$ gives $\mathbb E[T \mid X = x, C = c^\star] \leq 1/\tau$.
\end{proof}

\paragraph{Finite-cap variant.}
In practice, the acceptance loop is capped at $T_{\max}$ draws, and we fall back to an independent draw from $\mu(\cdot \mid x)$. Let $p_{\mathrm{fail}}(x, c) := (1 - \mu_C(c \mid x))^{T_{\max}}$ be the probability of no acceptance within $T_{\max}$ draws. If $|Y| \leq B$ almost surely, then this fallback produces an $O(p_{\mathrm{fail}})$ deviation in deployed value:
\begin{equation}
\big|V(\pi^{\uparrow}_{T_{\max}}) - V(\pi^\uparrow)\big|\ \leq\ 2B\,\mathbb E\!\left[\sum_{c=1}^K \pi_C(c \mid X)\,p_{\mathrm{fail}}(X, c)\right],
\label{eq:finite_cap_bound}
\end{equation}
where $\pi^{\uparrow}_{T_{\max}}$ is the capped sampler's induced deployed policy. This quantifies the trade-off between computational cost (at most $T_{\max}+1$ behavior draws) and approximation error ($p_{\mathrm{fail}}$ decaying exponentially in $T_{\max}$).

\subsection{Stochastic encoder extension}
\label{app:stoch_encoder}

The main text uses a deterministic encoder $C = g_\phi(A)$ to match the neural implementation at inference time. For completeness, we show that the value equivalence result extends to stochastic encoders.

Suppose the encoder is stochastic: $C \sim g_\phi(\cdot \mid A)$. The induced code-level behavior policy is $\mu_C(c \mid x) := \mathbb P(C = c \mid X = x)$, and the within-code behavior conditional is $\mu(a \mid x, c) := \mathbb P(A = a \mid X = x, C = c)$, both taken under the joint distribution over $(A, C)$. Behavior-consistent lifting is defined exactly as in~\eqref{eq:pi_up_def}.

\begin{theorem}[Value equivalence with stochastic encoding]
\label{thm:stochastic_equivalence}
Under consistency and ignorability and under Assumption~A3 applied to the joint $(A, C)$, for any $\pi_C \in \Pi_{C,\tau}$ the behavior-consistent deployed policy satisfies $V(\pi^\uparrow) = V_C(\pi_C)$, where $m_C(x, c) := \mathbb E[Y \mid X = x, C = c]$ under the joint distribution.
\end{theorem}

\begin{proof}
Identical to the proof of Corollary~\ref{cor:value_equivalence}, using conditioning on $(X, C)$ and Lemma~\ref{lem:code_mixture} for the joint distribution over $(A, C)$.
\end{proof}

\subsection{Proof of Theorem~\ref{thm:sig}}
\label{app:proof_sig}

Recall $\bar m(x) = \mathbb E[Y \mid X = x]$ and $m_C(x, c) = \mathbb E[Y \mid X = x, C = c]$.

\paragraph{Equivalent characterizations of $\mathrm{Sig}(\phi)$.}
By the tower property applied to $C$ conditional on $X$,
\begin{equation}
\mathbb E_{C \mid X}[m_C(X, C)] = \mathbb E_{C \mid X}[\mathbb E[Y \mid X, C]] = \mathbb E[Y \mid X] = \bar m(X).
\end{equation}
Hence
\begin{equation}
\mathrm{Sig}(\phi) = \mathbb E_X\!\left[\mathrm{Var}_{C \mid X}(m_C(X, C))\right] = \mathbb E_{X, C}\!\left[(m_C(X, C) - \bar m(X))^2\right],
\end{equation}
which is the first equality in~\eqref{eq:sig_mse}.

For the second equality, note that
\begin{equation}
\mathbb E[(Y - \bar m(X))^2] = \mathbb E_X[\mathrm{Var}(Y \mid X)],
\qquad
\mathbb E[(Y - m_C(X, C))^2] = \mathbb E_{X, C}[\mathrm{Var}(Y \mid X, C)].
\end{equation}
The law of total variance gives $\mathrm{Var}(Y \mid X) = \mathbb E_{C \mid X}[\mathrm{Var}(Y \mid X, C)] + \mathrm{Var}_{C \mid X}(\mathbb E[Y \mid X, C])$. Taking expectations over $X$ and rearranging,
\begin{equation}
\mathbb E_X[\mathrm{Var}(Y \mid X)] - \mathbb E_{X, C}[\mathrm{Var}(Y \mid X, C)] = \mathbb E_X\!\left[\mathrm{Var}_{C \mid X}(\mathbb E[Y \mid X, C])\right] = \mathrm{Sig}(\phi),
\end{equation}
which is the second equality in~\eqref{eq:sig_mse}.

\paragraph{Oracle value lower bound.}
Fix $x$ and let $Z := m_C(x, C)$ with $C \sim \mu_C(\cdot \mid x)$. Then $\mathbb E_{C \mid X = x}[Z] = \bar m(x)$ and
\begin{equation}
\mathrm{Var}(Z \mid X = x) = \mathbb E_{C \mid X = x}\!\left[(Z - \bar m(x))^2\right].
\end{equation}
Let
\begin{equation}
M_x := \max_{c:\,\mu_C(c\mid x)>0} m_C(x,c),
\qquad
m_x := \min_{c:\,\mu_C(c\mid x)>0} m_C(x,c).
\end{equation}
Since $|Y|\leq B$ almost surely, for $P_X$-almost every $x$ we have $m_C(x,c)\in[-B,B]$ whenever $\mu_C(c\mid x)>0$, and hence $Z\in[m_x,M_x]\subseteq[-B,B]$.

We use the elementary bounded-support inequality
\begin{equation}
\mathrm{Var}(Z \mid X = x)
\leq
\big(M_x - \mathbb E_{C \mid X = x}[Z]\big)
\big(\mathbb E_{C \mid X = x}[Z] - m_x\big).
\end{equation}
To see this, note that $(M_x-Z)(Z-m_x)\geq 0$ almost surely, so
\begin{equation}
Z^2 \leq (M_x+m_x)Z - M_xm_x.
\end{equation}
Taking expectations conditional on $X=x$ and subtracting $\mathbb E_{C \mid X = x}[Z]^2$ gives
\begin{equation}
\mathrm{Var}(Z \mid X = x)
\leq
\big(M_x - \mathbb E_{C \mid X = x}[Z]\big)
\big(\mathbb E_{C \mid X = x}[Z] - m_x\big).
\end{equation}
Moreover, since $Z\in[-B,B]$,
\begin{equation}
\mathbb E_{C \mid X = x}[Z] - m_x \leq 2B.
\end{equation}
Therefore,
\begin{equation}
\mathrm{Var}(Z \mid X = x)
\leq
2B\big(M_x - \mathbb E_{C \mid X = x}[Z]\big).
\end{equation}
Using $\mathbb E_{C \mid X = x}[Z] = \bar m(x)$ and rearranging yields
\begin{equation}
\max_{c:\,\mu_C(c\mid x)>0} m_C(x,c)
=
M_x
\geq
\bar m(x)
+
\frac{1}{2B}\,\mathrm{Var}(Z \mid X = x).
\end{equation}
Taking expectations over $X$ and using $V_C^\star(\phi) = \mathbb E_X[M_X]$ and $V(\mu) = \mathbb E_X[\bar m(X)]$ yields
\begin{equation}
V_C^\star(\phi) \geq V(\mu) + \frac{1}{2B}\,\mathbb E_X\!\left[\mathrm{Var}_{C \mid X}(m_C(X, C))\right] = V(\mu) + \frac{1}{2B}\,\mathrm{Sig}(\phi).
\end{equation}
\hfill$\square$

\subsection{Proof of Theorem~\ref{thm:ra}}
\label{app:proof_ra}

Fix $\pi_C \in \Pi_C$ and define
\begin{equation}
\widehat v(x) := \sum_{c = 1}^K \pi_C(c \mid x)\,\widehat m(x, c),
\qquad
v(x) := \sum_{c = 1}^K \pi_C(c \mid x)\,m_C(x, c),
\end{equation}
so that $\mathbb E_X[\widehat V_{\mathrm{RA}}(\pi_C)] = \mathbb E_X[\widehat v(X)]$ and $V_C(\pi_C) = \mathbb E_X[v(X)]$. Let $e_c(x) := m_C(x, c) - \widehat m(x, c)$ denote the code-level prediction residual.

\paragraph{Bias term.}
By linearity,
\begin{equation}
V_C(\pi_C) - \mathbb E_X[\widehat V_{\mathrm{RA}}(\pi_C)]
= \mathbb E_X\!\left[\sum_{c = 1}^K \pi_C(c \mid X)\,e_c(X)\right].
\end{equation}
Cauchy--Schwarz pointwise in $X$ gives
\begin{equation}
\left|\sum_{c = 1}^K \pi_C(c \mid X)\,e_c(X)\right|
\leq \sqrt{\sum_{c = 1}^K \pi_C(c \mid X)\,e_c(X)^2},
\end{equation}
using $\sum_c \pi_C(c \mid X) = 1$. Taking expectations and applying Jensen's inequality (since $\sqrt{\cdot}$ is concave) yields
\begin{equation}
\big|V_C(\pi_C) - \mathbb E_X[\widehat V_{\mathrm{RA}}(\pi_C)]\big|
\leq \sqrt{\mathbb E_X\!\left[\sum_{c = 1}^K \pi_C(c \mid X)\,e_c(X)^2\right]}.
\end{equation}
Since $\pi_C(c \mid X) \leq 1$ for all $c$,
\begin{equation}
\mathbb E_X\!\left[\sum_{c = 1}^K \pi_C(c \mid X)\,e_c(X)^2\right]
\leq \sum_{c = 1}^K \mathbb E_{P_X}[e_c(X)^2]
\leq \sum_{c = 1}^K \mathbb E_{Q_c^\phi}[e_c(X)^2] + \lambda_{\mathrm{ov}}\sum_{c = 1}^K D(P_X, Q_c^\phi),
\end{equation}
where the last inequality is the shift bound~\eqref{eq:overlap_shift_assump}. This gives the first term of~\eqref{eq:ra_bound_main}.

\paragraph{Sample fluctuation term.}
Treating $\widehat m$ as fixed and independent of the evaluation sample (by cross-fitting), we have $|\widehat v(x)| \leq B$ for all $x$, since $|\widehat m(x, c)| \leq B$ and $\sum_c \pi_C(c \mid x) = 1$. Hoeffding's inequality and a union bound over the finite class $\Pi_C$ imply that with probability at least $1 - \delta$,
\begin{equation}
\sup_{\pi_C \in \Pi_C}\ \big|\widehat V_{\mathrm{RA}}(\pi_C) - \mathbb E_X[\widehat V_{\mathrm{RA}}(\pi_C)]\big|
\leq B\sqrt{\frac{2\log(2|\Pi_C| / \delta)}{n}}.
\end{equation}
Combining with the bias bound via the triangle inequality yields~\eqref{eq:ra_bound_main}.
\hfill$\square$

\newpage
\section*{Evaluation benchmarks}

\paragraph{Evaluation design.}
Benchmarking off-policy learning against known ground-truth policy values requires counterfactual outcomes for actions that were not taken, which are unavailable in ordinary logged data. Controlled and semi-synthetic benchmarks are therefore standard for validating OPE/OPL methods against known ground truth, including in closely related large-action work using synthetic, supervised-to-bandit, and reward-simulator evaluations~\citep{saito2022off,saito2023off,saitopotec2025,kiyohara2025off}. We first evaluate \framework on two deliberately controlled benchmarks with complementary purposes: the synthetic DGP isolates the mechanisms analyzed by the theory and enables oracle diagnostics, while the semi-synthetic headline benchmark combines real article contexts and LLM-generated structured actions with deliberately recoverable reward-relevant structure. Additionally, to evaluate the complete method in a less controlled, externally defined generative setting, we use the public OfflinePrompts benchmark with real MovieLens-derived contexts, raw generated text without injected simulator features, and a released reward simulator trained on real user ratings.

\section{Synthetic diagnostics}
\label{app:synthetic_diagnostics}

Here, we give the full data-generating process (DGP) and evaluation protocol for the synthetic
diagnostics in Section~7. The goal is not to provide a broad benchmark, but to isolate the three
mechanisms of \framework under controlled conditions: behavior-consistent lifting (BCL), predictive action
codes, and overlap-aware code learning. All diagnostics use the same DGP and vary only the relevant
control parameter. The three parameters are: $\alpha_{\mathrm{dep}}$, which controls residual
within-code action effects; $\alpha_{\mathrm{pred}}$, which controls the between-code outcome signal;
and $\kappa_{\mathrm{ov}}$, which controls code-level overlap.

\subsection{Data-generating process}
\label{app:synthetic_dgp}

We generate logged triples $(X_i,A_i,Y_i)$ from a simulator with hidden context type $Z_i^\star$,
hidden simulator code $C_i^\star$, and structured numeric action $A_i$. The learner observes only
$(X_i,A_i,Y_i)$. The hidden variables are used only to construct oracle diagnostics.

\paragraph{Contexts.}
Let $K=8$. For each unit,
\begin{equation}
Z^\star \sim \mathrm{Unif}\{0,\ldots,K-1\}.
\end{equation}
The observed context is
\begin{equation}
X = (e_{Z^\star}+\xi_x,\; u_x),
\qquad
\xi_x\sim \mathcal N(0,0.05^2 I_K),
\qquad
u_x\sim \mathcal N(0,I_4),
\end{equation}
where $e_z\in\mathbb R^K$ is the one-hot vector for context type $z$. Thus $X\in\mathbb R^{K+4}$.
The hidden context type creates context-dependent optimal codes, while the low-dimensional observed
context keeps oracle evaluation simple.

\paragraph{Logging policy over simulator codes.}
We define a relative code
\begin{equation}
D = (C^\star-Z^\star)\bmod K.
\end{equation}
The logging policy samples $D$ according to
\begin{equation}
p_{\kappa_{\mathrm{ov}}}(d)
=
\kappa_{\mathrm{ov}}\mathbf 1\{d=1\}
+
(1-\kappa_{\mathrm{ov}})\frac{1}{K},
\qquad d\in\{0,\ldots,K-1\},
\end{equation}
and sets
\begin{equation}
C^\star = (Z^\star + D)\bmod K.
\end{equation}
Equivalently, the true code propensity is
\begin{equation}
\mu_{C^\star}(c\mid z)
=
p_{\kappa_{\mathrm{ov}}}\big((c-z)\bmod K\big).
\end{equation}
The relative-code construction makes the optimal code depend on the context while keeping all
contexts symmetric. The parameter $\kappa_{\mathrm{ov}}$ controls overlap: larger values concentrate
the logging policy on the plausible but suboptimal relative code $D=1$.

\paragraph{Actions.}
The action vector has five blocks,
\begin{equation}
A=(A_{\mathrm{abs}},A_{\mathrm{bal}},A_{\mathrm{dep}},A_{\mathrm{nuis}},A_{\mathrm{rand}}).
\end{equation}
The first block carries a recoverable signal for the true simulator code:
\begin{equation}
A_{\mathrm{abs}}
=
\lambda_{\mathrm{abs}} e_{C^\star}
+
\eta_{\mathrm{abs}},
\qquad
\eta_{\mathrm{abs}}\sim\mathcal N(0,\sigma_{\mathrm{abs}}^2 I_K).
\end{equation}
Here $e_{C^\star}$ is the one-hot vector for the hidden code $C^\star$. This block makes the
outcome-relevant code recoverable by predictive code learning.

The second block carries a signal for a more balanced auxiliary code. Define
\begin{equation}
P(D)=
\begin{cases}
0, & D\in\{0,1\},\\
1, & D\in\{2,7\},\\
2, & D\in\{3,6\},\\
3, & D\in\{4,5\},
\end{cases}
\qquad
Q\sim\mathrm{Bernoulli}(1/2),
\qquad
H^\star=2P(D)+Q.
\end{equation}
Then
\begin{equation}
A_{\mathrm{bal}}
=
\lambda_{\mathrm{bal}} e_{H^\star}
+
\eta_{\mathrm{bal}},
\qquad
\eta_{\mathrm{bal}}\sim\mathcal N(0,\sigma_{\mathrm{bal}}^2 I_K).
\end{equation}
The auxiliary code $H^\star$ creates a representation that
is slightly coarser than $C^\star$ but better balanced across contexts, matching the predictiveness--overlap trade-off targeted by Stage~1. In our experiments, we use $H^\star$ as \textit{oracle balanced codes}.

The third block creates residual within-code action variation:
\begin{equation}
A_{\mathrm{dep}}
=
t_{Z^\star}
+
\eta_{\mathrm{dep}},
\qquad
\eta_{\mathrm{dep}}\sim\mathcal N(0,\sigma_{\mathrm{dep}}^2 I_{d_{\mathrm{dep}}}),
\end{equation}
where $t_z\in\mathbb R^{d_{\mathrm{dep}}}$ is a fixed context-specific prototype. This block makes
the concrete action matter within a code and therefore tests whether deployment changes the
within-code action distribution.

The fourth block is high-dimensional nuisance geometry:
\begin{equation}
A_{\mathrm{nuis}}
=
\lambda_{\mathrm{nuis}}q_J
+
\eta_{\mathrm{nuis}},
\qquad
J\sim\mathrm{Unif}\{0,\ldots,K-1\},
\qquad
\eta_{\mathrm{nuis}}\sim\mathcal N(0,\sigma_{\mathrm{nuis}}^2I_{d_{\mathrm{nuis}}}),
\end{equation}
where $q_j\in\mathbb R^{d_{\mathrm{nuis}}}$ are fixed nuisance prototypes. The nuisance block has no
effect on the outcome but dominates Euclidean action geometry, so geometric clustering can fail.

Finally,
\begin{equation}
A_{\mathrm{rand}}\sim\mathcal N(0,I_2).
\end{equation}
This independent block is used only to construct random balanced codes.

\paragraph{Outcomes.}
Let
\begin{equation}
b(z)=0.2\sin(2\pi z/K)
\end{equation}
and define the between-code reward signal by
\begin{equation}
s(z,c)=v_{(c-z)\bmod K},
\qquad
(v_0,\ldots,v_7)=(1.00,0.65,0.15,-0.15,-0.30,-0.10,0.05,0.30).
\end{equation}
Thus the optimal relative code is $D=0$, while the logging policy favors the suboptimal but plausible
relative code $D=1$.

The within-code residual effect is
\begin{equation}
r(z,A)
=
-\frac{1}{d_{\mathrm{dep}}}\|A_{\mathrm{dep}}-t_z\|_2^2.
\end{equation}
Under behavior-consistent lifting,
\begin{equation}
\mathbb E[r(Z^\star,A)\mid Z^\star=z,C^\star=c]
=
-\sigma_{\mathrm{dep}}^2,
\end{equation}
which is constant in $(z,c)$. Hence this term does not change the best code-level policy under
behavior-consistent lifting; it only penalizes deployment rules that change the within-code action
distribution.

The observed outcome is
\begin{equation}
Y
=
b(Z^\star)
+
\alpha_{\mathrm{pred}}s(Z^\star,C^\star)
+
\alpha_{\mathrm{dep}}r(Z^\star,A)
+
\varepsilon,
\qquad
\varepsilon\sim\mathcal N(0,\sigma_y^2).
\end{equation}
This separates the two relevant failure modes. The parameter $\alpha_{\mathrm{pred}}$ controls the
value of learning predictive codes, while $\alpha_{\mathrm{dep}}$ controls the deployment loss from
non-behavior-consistent instantiation.

\subsection{Oracle quantities}
\label{app:synthetic_oracles}

For the true simulator code $C^\star$, the code regression is available in closed form:
\begin{equation}
m_{C^\star}(z,c)
=
b(z)
+
\alpha_{\mathrm{pred}}v_{(c-z)\bmod K}
-
\alpha_{\mathrm{dep}}\sigma_{\mathrm{dep}}^2.
\end{equation}
The behavior value is
\begin{equation}
V(\mu)
=
\frac{1}{K}
\sum_{z=0}^{K-1}
\sum_{c=0}^{K-1}
\mu_{C^\star}(c\mid z)m_{C^\star}(z,c).
\end{equation}
The unconstrained oracle code value is
\begin{equation}
V^\star_{C^\star}
=
\frac{1}{K}
\sum_{z=0}^{K-1}
\max_{c\in\{0,\ldots,K-1\}}
m_{C^\star}(z,c).
\end{equation}
For a support threshold $\tau$, the support-constrained oracle value is
\begin{equation}
V^\star_{C^\star,\tau}
=
\frac{1}{K}
\sum_{z=0}^{K-1}
\max_{c:\mu_{C^\star}(c\mid z)\ge \tau}
m_{C^\star}(z,c).
\end{equation}
For a deterministic code policy $\hat c(z)$, the expected number of behavior-system draws required
by behavior-consistent lifting is
\begin{equation}
\mathrm{Draws}(\hat c)
=
\frac{1}{K}
\sum_{z=0}^{K-1}
\frac{1}{\mu_{C^\star}(\hat c(z)\mid z)}.
\end{equation}
For learned or unsupervised representations $G=g(A)$, oracle quantities are estimated by a
noise-free Monte Carlo sample from the known simulator. We estimate $\mu_G(g\mid z)$ and
$m_G(z,g)$ by sample averages and then evaluate oracle policies by table lookup. This keeps the
oracle diagnostics cheap and avoids training additional oracle models.

\subsection{Code representations}
\label{app:synthetic_representations}

The diagnostics compare the following code representations.

\paragraph{Ground-truth oracle codes.}
The representation is $G=C^\star$. This is, in general, not available to the learner and serves as an upper benchmark for the information contained in the simulator code.

\paragraph{Oracle balanced codes.}
The representation is $G=H^\star$.  This is, in general, not available to the learner and serves as an optimistic target performance for the learned balanced codes.

\paragraph{Random balanced codes.}
We construct $G_{\mathrm{rand}}$ by quantile-binning a fixed random projection of $A_{\mathrm{rand}}$
into $K$ bins. These codes are balanced by construction but independent of the outcome-relevant action structure.

\paragraph{Geometric k-means codes.}
We fit k-means with $K$ clusters on the raw action vector $A$. Since $A_{\mathrm{nuis}}$ is high-dimensional and high-variance, this representation tests whether purely geometric action clustering can be dominated by nuisance variation.

\paragraph{Learned predictive codes.}
We learn codes with the Stage~1 prediction objective and set the overlap and marginal-balance
weights to zero. This representation tests whether outcome-predictive code learning can recover
useful action structure from $A$.

\paragraph{Learned predictive and balanced codes.}
We learn codes with the full Stage~1 objective, including the prediction loss, distributional-overlap
regularizer, and marginal code-balance regularizer. This representation tests whether overlap-aware
code learning preserves most predictive signal while improving finite-sample reliability.

\subsection{Diagnostic protocols}
\label{app:synthetic_protocols}

\paragraph{Deployment diagnostic.}
The deployment diagnostic fixes the representation to $C^\star$, uses the oracle code-level policy,
and uses the true support rule. It varies only the Stage~3 instantiation rule and
$\alpha_{\mathrm{dep}}$. We compare behavior-consistent lifting, a marginal decoder
$\mu(a\mid c)$, a learned Gaussian decoder $\hat\kappa(a\mid x,c)$, and a uniform within-code sampler. The
main metric is deployment loss,
\begin{equation}
V_C(\pi_C^\star)-V(\pi_\kappa),
\end{equation}
where $\pi_C^\star$ is the oracle code-level policy and $\pi_\kappa$ is the deployed action-level
policy induced by instantiation rule $\kappa$. This diagnostic isolates whether changing the
within-code action distribution creates a gap between code-level value and deployed value.

\paragraph{Predictiveness diagnostic.}
The predictiveness diagnostic fixes behavior-consistent lifting and uses an oracle Stage~2 policy for
each representation. It varies the code representation and $\alpha_{\mathrm{pred}}$. For a
representation $G$, the oracle policy is
\begin{equation}
\pi_G^\star(z)
\in
\arg\max_g m_G(z,g).
\end{equation}
We report oracle value, improvement over the behavior policy, representation regret relative to
$C^\star$, and prediction error. Representation regret is
\begin{equation}
V^\star_{C^\star}-V_G(\pi_G^\star).
\end{equation}
Because Stage~2 and Stage~3 are fixed to their oracle versions, differences in this diagnostic isolate
the information contained in the learned or constructed code representation, and thus, show the change in the oracle attainable policy value.

\paragraph{Overlap diagnostic.}
The overlap diagnostic fixes behavior-consistent lifting and learns the Stage~2 policy with a
regression-adjusted objective. It varies the representation, the support rule, and
$\kappa_{\mathrm{ov}}$. For each representation, we evaluate policies with and without the
$\tau$-constraint. The constrained policy selects only supported codes,
\begin{equation}
\hat g_\tau(z)
\in
\arg\max_{g:\hat\mu_G(g\mid z)\ge \tau}
\hat m_G(z,g),
\end{equation}
while the unconstrained policy maximizes $\hat m_G(z,g)$ over all codes. We report deployed value,
improvement over the behavior policy, RA error, unsupported-code rate, and expected draws. The
unsupported-code rate is
\begin{equation}
\mathbb E\left[
\mathbf 1\{\mu_G(\hat g(X)\mid X)<\tau\}
\right],
\end{equation}
and the RA error is
\begin{equation}
\left|
\hat V_{\mathrm{RA}}(\hat\pi_G)
-
V_G(\hat\pi_G)
\right|.
\end{equation}
This diagnostic isolates whether overlap-aware code learning and the $\tau$ support rule improve
reliability when code-level support becomes weak.

\newpage

\subsection{Default parameters}
\label{app:synthetic_parameters}

Unless otherwise stated, we use the parameters in Table~\ref{tab:synthetic_parameters}. We average
results over independent random seeds and report one standard error across seeds.

\begin{table}[h]
\centering
\small
\caption{Default parameters for the synthetic diagnostics.}
\label{tab:synthetic_parameters}
\begin{tabular}{ll}
\toprule
Quantity & Value \\
\midrule
Number of simulator codes and context types & $K=8$ \\
Observed context dimension & $K+4=12$ \\
Action block dimensions & $8+8+4+64+2$ \\
Outcome noise & $\sigma_y=0.25$ \\
Absolute-code block & $\lambda_{\mathrm{abs}}=0.8,\ \sigma_{\mathrm{abs}}=0.15$ \\
Balanced-code block & $\lambda_{\mathrm{bal}}=0.6,\ \sigma_{\mathrm{bal}}=0.25$ \\
Deployment block & $d_{\mathrm{dep}}=4,\ \sigma_{\mathrm{dep}}=0.25$ \\
Nuisance block & $d_{\mathrm{nuis}}=64,\ \lambda_{\mathrm{nuis}}=4.0,\ \sigma_{\mathrm{nuis}}=0.5$ \\
Default predictiveness strength & $\alpha_{\mathrm{pred}}=1$ \\
Default deployment-residual strength & $\alpha_{\mathrm{dep}}=0$ \\
Default overlap severity & $\kappa_{\mathrm{ov}}=0.4$ \\
Support threshold & $\tau=0.03$ \\
Deployment sweep & $\alpha_{\mathrm{dep}}\in\{0,0.25,0.5,1.0\}$ \\
Predictiveness sweep & $\alpha_{\mathrm{pred}}\in\{0,0.25,0.5,1.0,1.5\}$ \\
Overlap sweep & $\kappa_{\mathrm{ov}}\in\{0.4,0.6,0.8,0.9\}$ \\
\bottomrule
\end{tabular}
\end{table}

\section{Semi-synthetic benchmark construction}
\label{app:semisynth_construction}

We construct a semi-synthetic text benchmark with real article excerpts as contexts $X$ from the Upworthy Research Archive~\citep{matias2021upworthy} (CC BY 4.0), LLM-generated headlines as actions $A$, and rewards assigned by a controlled simulator. The benchmark evaluates the full method in a noisy structured-text setting under limited code-level overlap while deliberately retaining recoverable reward-relevant structure. All methods receive the same numerical representations (Appendix~\ref{app:semisynth_reward}), which makes comparisons less sensitive to whether generic text embeddings capture the simulated reward-relevant structure. OfflinePrompts (Appendix~\ref{app:offlineprompts}) separately evaluates raw generated text without injected simulator features.

\subsection{Structural data-generating process}
\label{app:semisynth_dgp}

The simulator follows the graph
$$
X \rightarrow Z^\star,\qquad
Z^\star \rightarrow C^\star,\qquad
(X,C^\star)\rightarrow A,\qquad
(Z^\star,C^\star,A)\rightarrow Y .
$$
For each unit $i$, we generate
$$
\begin{aligned}
Z_i^\star &= h_X(X_i),\\
C_i^\star &\sim \mu_{C^\star}(\cdot\mid Z_i^\star),\\
A_i &\sim G_{\mathrm{acc}}(\cdot\mid X_i,C_i^\star),\\
Y_i &= m^\star(Z_i^\star,C_i^\star)+\varepsilon_i,
\qquad \varepsilon_i\sim \mathcal N(0,\sigma^2).
\end{aligned}
$$
Learner inputs are specified in Appendix~\ref{app:semisynth_reward}; the annotators $h_X$ and $h_A$ are not provided to learning methods.

\begin{table}[t]
\centering
\small
\caption{Core notation for the semi-synthetic benchmark.}
\label{tab:ss_notation_final}
\begin{tabular}{ll}
\toprule
Symbol & Meaning \\
\midrule
$X$ & observed context text, here an Upworthy article excerpt \\
$A$ & observed action text, here a generated headline \\
$Y$ & observed scalar reward \\
$Z^\star\in\mathcal Z$ & hidden simulator context representation \\
$C^\star\in\mathcal C^\star$ & hidden simulator action code \\
$\mu_{C^\star}(c\mid z)$ & code-level logging policy \\
$G(\cdot\mid X,C^\star)$ & prompted headline generator \\
$G_{\mathrm{acc}}(\cdot\mid X,C^\star)$ & accepted generator after validation \\
$h_X$ & context annotator returning $Z^\star$ \\
$h_A$ & action validator returning realized code and faithfulness \\
$m^\star$ & structural mean reward \\
$M_i(c)$ & oracle code-level reward table for context $i$ \\
\bottomrule
\end{tabular}
\end{table}

\subsection{Simulator context and action spaces}
\label{app:semisynth_spaces}

We use Upworthy article excerpts as contexts. The real data provide natural language inputs, but outcomes are assigned by the simulator rather than by observed click rates. This gives realistic text contexts while retaining known counterfactual rewards.

The hidden simulator context representation is
$$
Z^\star=(T,U,N),
$$
where $T$ is the article topic, $U$ is the tone, and $N$ is the specificity need.

\begin{table}[t]
\centering
\small
\caption{Hidden simulator context representation $Z^\star=(T,U,N)$.}
\label{tab:ss_context_space}
\begin{tabularx}{\linewidth}{@{}ll>{\raggedright\arraybackslash}X@{}}
\toprule
Component & Meaning & Values \\
\midrule
$T$ & topic &
\texttt{public\_affairs}, \texttt{science\_health}, \texttt{entertainment}, \texttt{sports\_human} \\
$U$ & tone &
\texttt{informational}, \texttt{emotional}, \texttt{surprising} \\
$N$ & specificity need &
\texttt{low}, \texttt{high} \\
\bottomrule
\end{tabularx}
\end{table}

The hidden simulator action code is
$$
C^\star=(S,R),
$$
where $S$ is the headline strategy and $R$ is the tailoring level.

\begin{table}[t]
\centering
\small
\caption{Hidden simulator action code $C^\star=(S,R)$.}
\label{tab:ss_action_space}
\begin{tabular}{lll}
\toprule
Component & Meaning & Values \\
\midrule
$S$ & headline strategy &
\texttt{summary}, \texttt{curiosity}, \texttt{emotion}, \texttt{direct} \\
$R$ & tailoring level &
\texttt{generic}, \texttt{tailored} \\
\bottomrule
\end{tabular}
\end{table}

A \texttt{summary} headline states the core content directly. A \texttt{curiosity} headline withholds a key detail to induce curiosity without being misleading. An \texttt{emotion} headline foregrounds affective framing. A \texttt{direct} headline addresses the reader or frames the article as personally relevant. A \texttt{generic} headline avoids article-specific named entities or concrete details, while a \texttt{tailored} headline includes article-specific entities, details, or concrete context.

\subsection{Reward-relevant strategy structure}
\label{app:semisynth_strategy_structure}

For each topic--tone pair $(T,U)$, we define a full simulator ranking of headline strategies
$$
S_1^\star(T,U),S_2^\star(T,U),S_3^\star(T,U),S_4^\star(T,U),
$$
where $S_1^\star(T,U)$ is best and $S_4^\star(T,U)$ is worst. Table~\ref{tab:ss_strategy_ranking} gives the ranking used in all experiments.

\begin{table}[t]
\centering
\small
\caption{Simulator ranking of headline strategies. Entries are ordered from best to worst.}
\label{tab:ss_strategy_ranking}
\begin{tabularx}{\linewidth}{@{}l*{3}{>{\raggedright\arraybackslash}X}@{}}
\toprule
Topic $T$ & $U=\texttt{informational}$ & $U=\texttt{emotional}$ & $U=\texttt{surprising}$ \\
\midrule
\texttt{public\_affairs}
& summary $>$ direct $>$ emotion $>$ curiosity
& direct $>$ emotion $>$ summary $>$ curiosity
& curiosity $>$ direct $>$ summary $>$ emotion \\
\texttt{science\_health}
& summary $>$ curiosity $>$ direct $>$ emotion
& direct $>$ emotion $>$ summary $>$ curiosity
& curiosity $>$ summary $>$ direct $>$ emotion \\
\texttt{entertainment}
& direct $>$ curiosity $>$ emotion $>$ summary
& emotion $>$ curiosity $>$ direct $>$ summary
& curiosity $>$ emotion $>$ direct $>$ summary \\
\texttt{sports\_human}
& summary $>$ emotion $>$ direct $>$ curiosity
& emotion $>$ direct $>$ summary $>$ curiosity
& curiosity $>$ emotion $>$ direct $>$ summary \\
\bottomrule
\end{tabularx}
\end{table}

The logging-preferred tailoring level is
$$
R^\star(N)=
\begin{cases}
\texttt{generic}, & N=\texttt{low},\\
\texttt{tailored}, & N=\texttt{high}.
\end{cases}
$$
For context type $z=(T,U,N)$, the reward-maximizing codes are
$$
\arg\max_{c}m^\star(z,c)
=\big\{(S_1^\star(T,U),r):r\in\{\texttt{generic},\texttt{tailored}\}\big\}.
$$

\subsection{Structural reward model}
\label{app:semisynth_reward}

For $z=(T,U,N)$ and $c=(s,r)$, the reward depends on the context-specific rank of the strategy $s$:
\begin{equation}
m^\star(z,c)
=
b(z)
+
\theta_{\operatorname{rank}(s;T,U)} .
\end{equation}
Here $\operatorname{rank}(s;T,U)\in\{1,2,3,4\}$ is given by the strategy lookup table in Table~\ref{tab:ss_strategy_ranking}. The rank effects are
\begin{equation}
(\theta_1,\theta_2,\theta_3,\theta_4)
=
(1.00,0.65,0.15,-0.15).
\end{equation}
The additive context baseline is deterministic and does not affect the optimal action code:
\begin{equation}
b(z_i)
=
0.2\sin\!\left(
2\pi\,r(z_i)/|\mathcal Z_{\mathrm{obs}}|
\right),
\end{equation}
where $r(z_i)$ is the index of the observed hidden context label among the labels in the constructed benchmark. In the reported benchmark there is no additional outcome noise, so
\begin{equation}
Y_i
=
m^\star(Z_i^\star,C_i^\star).
\end{equation}
Policy values are evaluated using the same noise-free oracle reward table.

\paragraph{Observed numerical representations.}
The text learner operates on numerical representations $(\tilde X_i,\tilde A_i)$ that augment base text embeddings with controlled simulator-derived features: an exact context-label encoding and noisy action-code and strategy signals, alongside nuisance and random features. For contexts,
\begin{equation}
\tilde X_i
=
\big[
\phi_X(X_i),\,
3.0\,e_{Z_i^\star}
\big],
\end{equation}
where $e_{Z_i^\star}$ is a one-hot encoding of the hidden context label. For actions,
\begin{equation}
\tilde A_i
=
\big[
\phi_A(A_i),\,
0.8\,e_{C_i^\star}+\varepsilon^{(C)}_i,\,
0.6\,e_{H_i^\star}+\varepsilon^{(H)}_i,\,
4.0\,v_{J_i}+\varepsilon^{(J)}_i,\,
\varepsilon^{(R)}_i
\big],
\end{equation}
where $H_i^\star$ is the four-way strategy component of $C_i^\star$, $v_{J_i}$ is a random normalized nuisance prototype, and the noise scales are $0.15$, $0.25$, and $0.5$ for the $C^\star$, strategy, and nuisance blocks, respectively. The nuisance block has dimension $64$ and the final random tail has dimension $2$. This construction preserves the same generated texts, logging policy, and candidate pool, but makes explicit which reward-relevant and nuisance signals are available to representation learning.

\subsection{Code-level logging policy}
\label{app:semisynth_logging}

The logged simulator code is sampled from
$$
\mu_{C^\star}(s,r\mid z)
=
\mu_S(s\mid T,U)\mu_R(r\mid N).
$$
This factorization reflects that headline strategy depends on topic and tone, while tailoring depends on specificity need.

For strategy logging, we use an orientation parameter
$$
d\in\{\mathrm{optimal},\mathrm{suboptimal}\}
$$
and a concentration parameter $\kappa_S\in[0,1]$. Define
$$
s_{\mathrm{pref}}^d(T,U)=
\begin{cases}
S_1^\star(T,U), & d=\mathrm{optimal},\\
S_3^\star(T,U), & d=\mathrm{suboptimal}.
\end{cases}
$$
Then
$$
\mu_S(s\mid T,U)
=
\kappa_S\1\{s=s_{\mathrm{pref}}^d(T,U)\}
+
(1-\kappa_S)\frac{1}{|\mathcal S|}.
$$
Thus, $\kappa_S=0$ gives uniform strategy logging, while larger $\kappa_S$ concentrates mass on the preferred logging strategy. In the suboptimal-oriented regime, the historical policy favors a plausible but non-optimal strategy while still assigning nonzero probability to the optimal one.

For tailoring, we use
$$
\mu_R(r\mid N)
=
\kappa_R\1\{r=R^\star(N)\}
+
(1-\kappa_R)\frac{1}{|\mathcal R|}.
$$
Larger $\kappa_R$ means that the logging policy more often chooses the tailoring level appropriate for the context.

With four strategies, $\kappa_S=0.60$ assigns probability $0.70$ to the preferred strategy and probability $0.10$ to each other strategy. With two tailoring levels, $\kappa_R=0.50$ assigns probability $0.75$ to the appropriate tailoring level and probability $0.25$ to the other level.

\subsection{Prompt-based construction}
\label{app:semisynth_prompts}

The benchmark uses three prompt-based components: a context annotator $h_X$, a headline generator $G$, and an action validator $h_A$. All prompts return structured outputs.

\begin{tcolorbox}[title={Prompt 1: Context annotator $h_X$},colback=white,colframe=black!70,boxrule=0.5pt]
\small
Read the article excerpt and return JSON with keys \texttt{topic}, \texttt{tone}, and \texttt{specificity\_need}. Use only the following values:
\begin{itemize}
    \item \texttt{topic} $\in$ \{\texttt{public\_affairs}, \texttt{science\_health}, \texttt{entertainment}, \texttt{sports\_human}\},
    \item \texttt{tone} $\in$ \{\texttt{informational}, \texttt{emotional}, \texttt{surprising}\},
    \item \texttt{specificity\_need} $\in$ \{\texttt{low}, \texttt{high}\}.
\end{itemize}
Set \texttt{specificity\_need = high} if an effective headline should mention article-specific entities, people, places, or concrete details; otherwise set it to \texttt{low}. Output only JSON.
\end{tcolorbox}

\begin{tcolorbox}[title={Prompt 2: Headline generator $G$},colback=white,colframe=black!70,boxrule=0.5pt]
\small
Read the article excerpt. Write exactly one Upworthy-style headline that is faithful to the excerpt and does not invent facts.

Required strategy: \texttt{\{summary | curiosity | emotion | direct\}}.

Required tailoring level: \texttt{\{generic | tailored\}}.

Definitions:
\begin{itemize}
    \item \texttt{summary}: state the core content of the article directly.
    \item \texttt{curiosity}: withhold a key detail to create curiosity without being misleading.
    \item \texttt{emotion}: foreground emotional or affective framing.
    \item \texttt{direct}: address the reader or make the article personally relevant.
    \item \texttt{generic}: avoid article-specific named entities or concrete details.
    \item \texttt{tailored}: include article-specific entities, details, or concrete context.
\end{itemize}
Keep the headline concise. If an observed headline is provided, use it only as a stylistic anchor and do not copy it verbatim. Output only the headline.
\end{tcolorbox}

\begin{tcolorbox}[title={Prompt 3: Action validator $h_A$},colback=white,colframe=black!70,boxrule=0.5pt]
\small
Given an article excerpt and a headline, return JSON with keys \texttt{strategy}, \texttt{tailoring}, and \texttt{faithful}. Use only:
\begin{itemize}
    \item \texttt{strategy} $\in$ \{\texttt{summary}, \texttt{curiosity}, \texttt{emotion}, \texttt{direct}\},
    \item \texttt{tailoring} $\in$ \{\texttt{generic}, \texttt{tailored}\},
    \item \texttt{faithful} $\in$ \{0,1\}.
\end{itemize}
Set \texttt{faithful = 1} only if the headline is supported by the excerpt and does not introduce unsupported claims. Output only JSON.
\end{tcolorbox}

\subsection{Accepted action generator}
\label{app:semisynth_accepted_generator}

The prompt $G(\cdot\mid X,C^\star)$ asks the language model to generate a headline with the requested simulator code. Since prompt compliance is not guaranteed, we validate the realized code using $h_A$. For intended code $c$, write
$$
h_A(X,A)=(\tilde C^\star,\mathrm{faithful}).
$$
We accept a generated headline only if
$$
\tilde C^\star=c
\qquad\text{and}\qquad
\mathrm{faithful}=1.
$$
This defines the accepted generator
$$
G_{\mathrm{acc}}(\cdot\mid X,c)
=
G(\cdot\mid X,c,\ h_A(X,A)=(c,1)).
$$
Equivalently, $G_{\mathrm{acc}}$ is implemented by rejection sampling from $G$. This validation step is performed before assigning outcomes and does not use $Y$. Its role is to ensure that the observed text action $A_i$ is semantically aligned with the simulator code $C_i^\star$ whose reward is assigned.

In implementation, we use a fixed retry budget $L_{\max}$. If no valid headline is obtained after $L_{\max}$ attempts, the unit is discarded.

\subsection{Logged data and oracle evaluation}
\label{app:semisynth_oracle_evaluation}

After construction, all learners use the representations specified in Appendix~\ref{app:semisynth_reward} for the logged data
$$
\mathcal D_{\mathrm{log}}=\{(X_i,A_i,Y_i)\}_{i=1}^n.
$$
The simulator reward function and oracle counterfactual rewards are not provided to learners.

For the main benchmark, we store the oracle code-level reward table
$$
M_i(c)=m^\star(Z_i^\star,c),
\qquad c\in\mathcal C^\star.
$$
The oracle code policy is
$$
c_i^\star
=
\argmax_{c\in\mathcal C^\star} M_i(c),
$$
and its value is evaluated without outcome noise.

For any deployed text action $\tilde A_i$, we first validate the realized simulator code:
$$
h_A(X_i,\tilde A_i)=(\tilde C_i^\star,\mathrm{faithful}_i).
$$
Only valid and faithful deployed headlines are evaluated. For valid headlines, the noise-free deployed reward is
$$
v_i(\tilde A_i)
=
m^\star(Z_i^\star,\tilde C_i^\star).
$$
In the main code-sufficient benchmark, this reduces to
$$
v_i(\tilde A_i)
=
M_i(\tilde C_i^\star).
$$
Thus, the benchmark has an exact code-level oracle while deployed policies return generated headlines.

\section{OfflinePrompts benchmark}
\label{app:offlineprompts}

Beyond the two controlled benchmarks, we evaluate the complete method on the public MovieLens task from OfflinePrompts~\citep{kiyohara2025off}, providing additional validation in an externally defined generative setting. Contexts combine the released embeddings of real MovieLens users and movies; actions are raw movie-related sentences generated by the frozen Mistral-7B-Instruct-v0.2 model, with no injected simulator features. Rewards come from the released personalized DistilBERT reward simulator trained on real user ratings. \framework learns from context, generated text, and reward without observing the generating prompt.

\paragraph{Behavior system and logged data.}
For reproducible and computationally efficient evaluation, we pre-generate the Mistral outputs under the benchmark prompt set and reuse the same generated outputs across methods. We use a fixed subset of 500 movies and 100 prompts. For each movie--prompt pair, we generate one sentence by greedy decoding with a maximum of 30 new tokens and the released prompt formatter. The behavior system samples a prompt from a context-dependent softmax distribution and returns the corresponding pre-generated sentence. Its scores are bilinear in the context and prompt embeddings; we calibrate the softmax concentration on training contexts so that the top 50 prompts carry mean probability mass $0.90$. Prompt identities and their logging probabilities are hidden from \framework, which receives only context, sentence embedding, and reward. Each of five seeds produces 10,000 logged observations, adding independent Gaussian reward noise with standard deviation $0.05$.

\paragraph{Representation and evaluation.}
We follow the released sentence-embedding procedure, projecting pooled Mistral hidden states to 20 dimensions with the supplied PCA matrix. We use the released user split, with disjoint training and test halves, and evaluate all methods on the same 10,000 sampled test contexts. The movie and prompt catalog is shared across the split. All learned policies and generated-sentence selections are frozen before evaluation with the reward simulator. Reported values average its expected rewards, without added observation noise. Behavior draws use the same behavior system for every method; \framework uses a cap of $T_{\max}=20$ and an independent behavior draw on fallback.

\paragraph{Training configuration.}
We use $K=8$, $\tau=0.03$, two hidden layers of width~64, batch size~32, learning rate $10^{-4}$, and 50 Stage~1 epochs. The proposed configuration uses $\beta=10^{-4}$ and $\gamma=0.2$; predictive-only codes set both to zero. Stage~2 fits a ridge outcome model with context--code interactions and a multinomial logistic propensity model. RA selects the highest predicted reward among supported codes. Standard errors describe variability across these runs on the fixed test contexts and pre-generated outputs.

\newpage

\section{Additional results}\label{app:add_experiments}

\subsection{Synthetic diagnostics}

Figure~\ref{fig:synthetic_diagnostics_complete} presents the complete synthetic
diagnostics. Its columns vary within-code reward effects
($\alpha_{\mathrm{dep}}$, left), predictive signal
($\alpha_{\mathrm{pred}}$, middle), and logging-policy concentration
($\kappa_{\mathrm{ov}}$, right). The left legend lists deployment rules;
the right legend lists code representations for the middle and right columns.
Predictiveness and overlap curves show means with standard errors over 5
seeds. Deployment values are computed analytically, so no standard errors
are shown. The protocols are described in Appendix~\ref{app:synthetic_protocols}.

\begin{figure}[htbp]
\centering
\includegraphics[width=\linewidth]{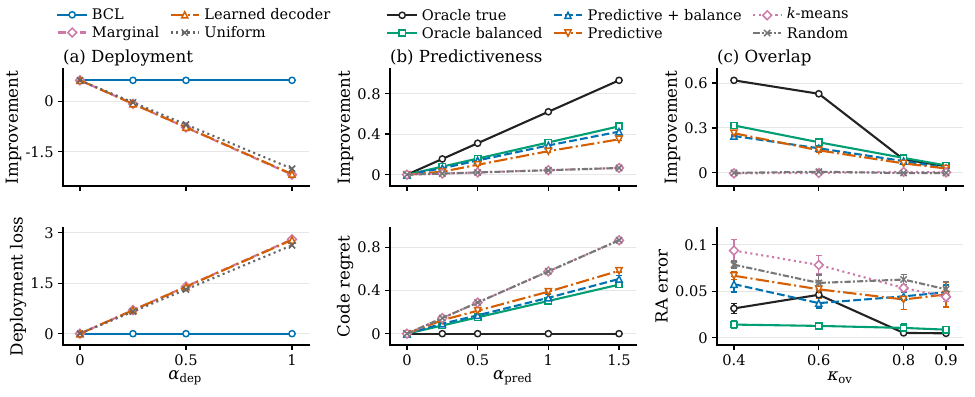}
\caption{\textbf{Complete synthetic diagnostics.}
Top: policy improvement (higher is better).
Bottom: deployment loss, code regret, and RA estimation error, respectively
(lower is better).}
\label{fig:synthetic_diagnostics_complete}
\end{figure}

\FloatBarrier
\subsubsection{Deployment}

In the left column of Figure~\ref{fig:synthetic_diagnostics_complete},
ground-truth codes and the oracle code-level policy are fixed while the
instantiation rule varies. Table~\ref{tab:synthetic_deployment_diagnostic}
reports the strongest residual-effect setting, $\alpha_{\mathrm{dep}}=1$,
under the default synthetic DGP. All rows use the ground-truth code $C^\star$,
the oracle code-level policy, and the true support rule; rows differ only in
the Stage~3 instantiation rule. Values are computed analytically under the
known DGP, so no standard errors are reported.

Deployment loss is $V_C(\pi_C^\star)-V(\pi_\kappa)$, and absolute deployment
error is $\left|V(\pi_\kappa)-V_C(\pi_C^\star)\right|$.
Behavior-consistent lifting (BCL) preserves the code-level value, while
alternative instantiation rules change the within-code action distribution
and therefore incur deployment loss.

\begin{table}[htbp]
\centering
\caption{\textbf{Deployment diagnostic.} Default synthetic DGP at $\alpha_{\mathrm{dep}}=1$; analytically computed values.}
\label{tab:synthetic_deployment_diagnostic}
\small
\setlength{\tabcolsep}{4pt}
\renewcommand{\arraystretch}{1.15}
\begin{tabular}{lcccc}
\toprule
Deployment rule
& \shortstack{Policy improvement\\$V(\pi)-V(\mu)$}
& \shortstack{Value\\$V(\pi)$}
& \shortstack{Deployment loss\\$V_C(\pi_C)-V(\pi_\kappa)$}
& \shortstack{Abs. deployment\\error} \\
\midrule
BCL / Ours   &  0.620 &  0.938 & 0.000 & 0.000 \\
Marginal     & -2.188 & -1.870 & 2.808 & 2.808 \\
Learned dec. & -2.170 & -1.853 & 2.790 & 2.790 \\
Uniform      & -2.010 & -1.693 & 2.630 & 2.630 \\
\bottomrule
\end{tabular}
\end{table}

\FloatBarrier
\subsubsection{Predictiveness}

In the middle column of Figure~\ref{fig:synthetic_diagnostics_complete},
oracle Stage~2 optimization and behavior-consistent lifting are fixed while
the representation varies.
Table~\ref{tab:synthetic_predictiveness_diagnostic} reports the default
synthetic DGP at $\alpha_{\mathrm{pred}}=1$. All rows use an oracle Stage~2
code-level policy and behavior-consistent lifting; rows differ only in the
Stage~1 code representation. Values are reported as mean (standard error)
over random seeds.

Oracle true codes give the learner-inaccessible upper benchmark, while
oracle balanced codes are a designed support-aware proxy. Policy improvement
is $V_G^\star-V(\mu)$, where $V_G^\star$ is the value of the oracle code-level
policy for representation $G$. Code regret is
$V_{C^\star}^\star-V_G^\star$. Prediction error is the noise-free grouped
error $\mathbb E[(Y-m_G(X,G))^2]$.
Because Stage~2 is oracle for every row, differences isolate the information
contained in the code representation.

\begin{table}[htbp]
\centering
\caption{\textbf{Predictiveness diagnostic.} Default synthetic DGP at $\alpha_{\mathrm{pred}}=1$; mean (standard error) over random seeds.}
\label{tab:synthetic_predictiveness_diagnostic}
\small
\setlength{\tabcolsep}{4pt}
\renewcommand{\arraystretch}{1.15}
\begin{tabular}{lcccc}
\toprule
Code representation
& Value
& \shortstack{Policy improvement\\$V(\pi)-V(\mu)$}
& \shortstack{Code regret\\$V_{C^\star}^\star - V_G^\star$}
& \shortstack{Prediction error\\$\mathbb E[(Y-m_G(X,G))^2]$} \\
\midrule
Oracle true codes     & 1.000 (0.000) & 0.620 (0.000) & 0.000 (0.000) & 0.000 (0.000) \\
Oracle balanced codes & 0.698 (0.000) & 0.318 (0.000) & 0.302 (0.000) & 0.012 (0.000) \\
\midrule
Pred.+bal. codes      & 0.668 (0.036) & 0.288 (0.036) & 0.332 (0.036) & 0.088 (0.001) \\
Pred. codes           & 0.612 (0.015) & 0.232 (0.015) & 0.388 (0.015) & 0.084 (0.003) \\
K-means               & 0.424 (0.006) & 0.044 (0.006) & 0.576 (0.006) & 0.148 (0.001) \\
Random                & 0.424 (0.004) & 0.044 (0.004) & 0.576 (0.004) & 0.148 (0.001) \\
\bottomrule
\end{tabular}
\end{table}

\FloatBarrier
\subsubsection{Overlap}

In the right column of Figure~\ref{fig:synthetic_diagnostics_complete},
behavior-consistent lifting and the support constraint are fixed while
representations are compared under increasing logging-policy concentration.
Together, the overlap panels distinguish attainable policy improvement from
value-estimation accuracy.

Table~\ref{tab:synthetic_overlap_diagnostic} reports the default synthetic
DGP at $\kappa_{\mathrm{ov}}=0.9$. Rows use the same code representations as
in Table~\ref{tab:synthetic_predictiveness_diagnostic}, while Stage~2 is
learned with a regression-adjusted objective and all rows use the $\tau$
support constraint and behavior-consistent lifting. Values are reported as
mean (standard error) over random seeds.

Oracle true codes give the learner-inaccessible upper benchmark under the
support constraint, while oracle balanced codes are a designed support-aware
proxy. RA error is
$\left|\widehat V_{\mathrm{RA}}(\widehat\pi_C)-V_C(\widehat\pi_C)\right|$.
Expected draws is $\mathbb E[1/\mu_C(\widehat c(X)\mid X)]$ and measures the
cost of behavior-consistent lifting. All rows are support constrained, so
unsupported-code rates are zero by construction and omitted.

\begin{table}[htbp]
\centering
\caption{\textbf{Overlap diagnostic.} Default synthetic DGP at $\kappa_{\mathrm{ov}}=0.9$; mean (standard error) over random seeds.}
\label{tab:synthetic_overlap_diagnostic}
\small
\setlength{\tabcolsep}{4pt}
\renewcommand{\arraystretch}{1.15}
\begin{tabular}{lcccc}
\toprule
Code representation
& Value
& \shortstack{Policy improvement\\$V(\pi)-V(\mu)$}
& \shortstack{RA error\\$|\widehat V_{\mathrm{RA}}-V|$}
& \shortstack{Expected\\draws} \\
\midrule
Oracle true codes     & 0.649 (0.000) &  0.044 (0.000) & 0.005 (0.000) & 1.10 \\
Oracle balanced codes & 0.653 (0.000) &  0.048 (0.000) & 0.009 (0.004) & 2.16 \\
\midrule
Pred.+bal. codes      & 0.635 (0.008) &  0.030 (0.008) & 0.049 (0.008) & 3.63 \\
Pred. codes           & 0.633 (0.004) &  0.028 (0.004) & 0.047 (0.014) & 3.09 \\
K-means               & 0.607 (0.003) &  0.002 (0.003) & 0.044 (0.005) & 7.96 \\
Random                & 0.603 (0.002) & -0.002 (0.002) & 0.052 (0.008) & 8.17 \\
\bottomrule
\end{tabular}
\end{table}

\FloatBarrier
\subsection{Sensitivity to loss parameters}

Figure~\ref{fig:semisynth_ablations} shows sensitivity to the loss parameters
$\beta$ (top) and $\gamma$ (bottom) on the semi-synthetic dataset.
The left panels report policy improvement (higher is better), and the right
panels report expected draws (lower is better).

\begin{figure}[htbp]
\centering
\begin{subfigure}{0.85\linewidth}
    \centering
    \includegraphics[width=\linewidth]{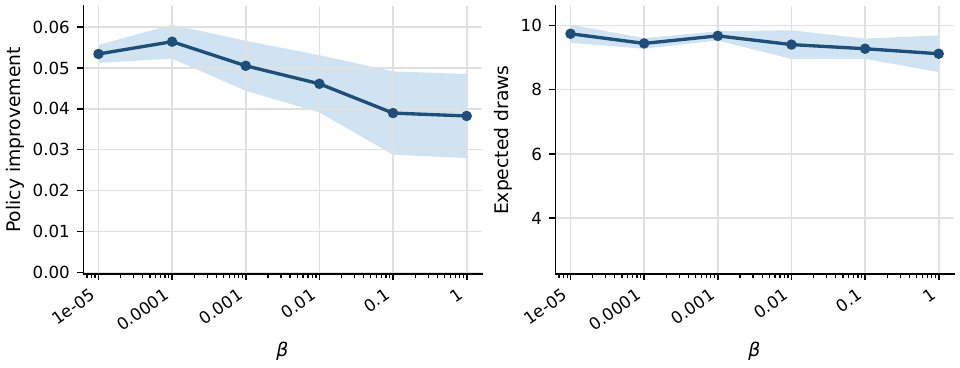}
    \phantomcaption\label{fig:top}
\end{subfigure}

\medskip

\begin{subfigure}{0.85\linewidth}
    \centering
    \includegraphics[width=\linewidth]{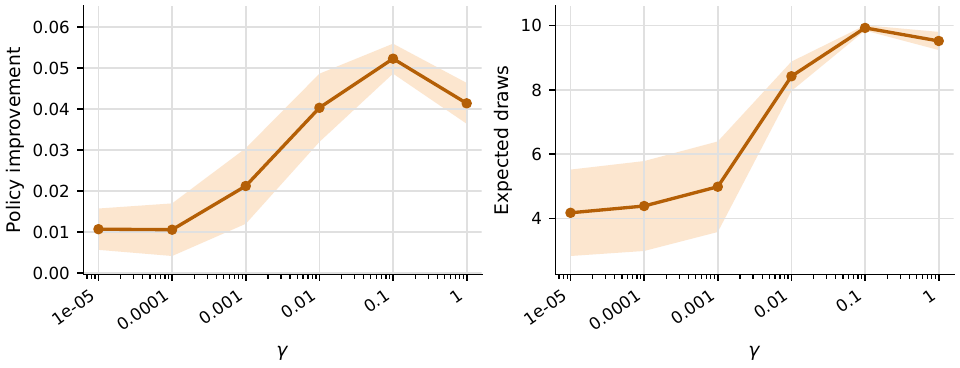}
    \phantomcaption\label{fig:bottom}
\end{subfigure}
\caption{\textbf{Sensitivity to loss parameters on the semi-synthetic dataset.}
$\beta$ varies in the top row and $\gamma$ in the bottom row.
Left: policy improvement; right: expected draws.}
\label{fig:combined}
\label{fig:semisynth_ablations}
\end{figure}
\FloatBarrier

\newpage

\section{Implementation details}
\label{app:implementation}
\begin{algorithm}[h]
\caption{PolicyLift (implementation-aligned)}
\label{alg:policylift_impl_long}
\KwIn{
Logged data $\mathcal D=\{(X_i,A_i,Y_i)\}_{i=1}^n$;
code encoder $g_\phi$ (action branch), context encoder $\psi$, outcome head $f_\theta$;
number of codes $K$;
weights $\beta\ge 0,\gamma\ge 0$;
threshold $\tau\ge 0$;
acceptance cap $T_{\max}$.
}
\KwOut{Deployed policy sampler $\widehat\pi^\uparrow(\cdot\mid x)$.}

\BlankLine
\textbf{Stage 1: Learn code encoder and predictor}\;
Initialize $(\phi,\psi,\theta)$\;
\While{not converged}{
Compute logits $\ell_i \leftarrow \text{enc}_a(A_i)$ and straight-through Gumbel-hard codes
$\tilde C_i \leftarrow \text{GumbelSoftmax}(\ell_i;\texttt{hard=true})$\;
Compute context embedding $Z_i \leftarrow \psi(X_i)$\;
$\widehat{\mathcal L}_{\text{pred}} \leftarrow \frac{1}{n}\sum_i \big(Y_i-f_\theta([Z_i,\tilde C_i])\big)^2$\;
$\widehat{\mathcal L}_{\text{ov}} \leftarrow \sum_{c=1}^K \widehat D\!\left(\widehat P_{Z},\widehat P_{Z\mid \tilde C=c}\right)$
\tcp*{stop-gradient through $Z$ in this term}
$\widehat{\mathcal L}_{\text{cb}} \leftarrow \mathrm{KL}(\widehat p_C\|\mathrm{Unif}(\{1,\dots,K\}))$, with $\widehat p_C=\frac{1}{n}\sum_i \tilde C_i$\;
Update $(\phi,\psi,\theta)$ to minimize
$\widehat{\mathcal L}_{\text{pred}}+\beta\,\widehat{\mathcal L}_{\text{ov}}+\gamma\,\widehat{\mathcal L}_{\text{cb}}$\;
}
Freeze Stage-1 model; define hard codes for Stage 2 by $C_i=\arg\max g_\phi(A_i)$\;

\BlankLine
\textbf{Stage 2: Learn outcome head, propensity proxy, and code policy}\;
Fit $\widehat m(x,c)\approx \mathbb E[Y\mid X=x,C=c]$ on $(X_i,C_i,Y_i)$
(neural head; optional initialization from Stage-1 head)\;
Fit $\widehat\mu_C(c\mid x)$ via multiclass cross-entropy on $(X_i,C_i)$\;
Define $\pi_C(x)\in\arg\max_{c:\,\widehat\mu_C(c\mid x)\ge\tau}\widehat m(x,c)$\;

\BlankLine
\textbf{Stage 3: Deployment at context $x$}\;
Select target code $c^\star=\pi_C(x)$\;

\For{$t=1,\dots,T_{\max}$}{
Sample $a_t\sim\mu(\cdot\mid x)$\;
\If{$\arg\max g_\phi(a_t)=c^\star$}{
\KwRet{$a_t$}
}
}
\KwRet{sample from behavior $\mu(\cdot\mid x)$} \tcp*{fallback after cap}
\end{algorithm}

\textbf{Synthetic diagnostics.}
 We follow the data generation and evaluation as described in ~\ref{app:synthetic_diagnostics}. Table~\ref{tab:synthetic_impl_params} gives the implementation choices for the reported runs. We provide code of our method in \url{https://github.com/JSchweisthal/PolicyLift}

\begin{table}[h]
\centering
\small
\caption{Implementation details for the synthetic diagnostics.}
\label{tab:synthetic_impl_params}
\begin{tabular}{p{0.30\linewidth}p{0.62\linewidth}}
\toprule
Component & Choice \\
\midrule
Code space & $K=8$ \\
Context/action dimensions & $12$ and $8+8+4+64+2$ \\
Support threshold & $\tau=0.03$ \\
Deployment diagnostic & oracle codes and oracle code policy; vary only deployment rule \\
Predictiveness diagnostic & BCL and oracle Stage~2; vary code representation and $\alpha_{\mathrm{pred}}$ \\
Overlap diagnostic & learned RA Stage~2 with support rule; vary representation and $\kappa_{\mathrm{ov}}$ \\
\framework{} Stage~1 & MLP hidden sizes $(64,64)$, latent dimension $8$, Gumbel temperature $0.5$ \\
Optimization & $50$ epochs, batch size $32$, learning rate $10^{-4}$ \\
Main regularization & $\beta=0.0001$, $\gamma=0.2$ for Pred.+bal. \\
Reported uncertainty & mean with one standard error across seeds \\
\bottomrule
\end{tabular}
\end{table}

\textbf{Semi-synthetic headline benchmark.}
The semi-synthetic experiments use the logged headlines, validation procedure, and candidate pool from Appendix~\ref{app:semisynth_construction}. Table~\ref{tab:semisynth_impl_params} gives the implementation choices for the reported runs.

\begin{table}[h]
\centering
\small
\caption{Implementation details for the semi-synthetic headline benchmark.}
\label{tab:semisynth_impl_params}
\begin{tabular}{p{0.30\linewidth}p{0.62\linewidth}}
\toprule
Component & Choice \\
\midrule
Data & $905$ logged actions, $2513$ generated actions, $5$ seeds \\
Data generator model & GPT-5.4-mini\\
Code space and support & $K=8$, $\tau=0.03$ \\
Reward & strategy-rank lookup with $(1.00,0.65,0.15,-0.15)$ and context baseline scale $0.2$ \\
Context representation & sentence-transformers/all-MiniLM-L6-v2; append $3.0\,e_{Z^\star}$ \\
Action representation & sentence-transformers/all-MiniLM-L6-v2; append noisy $C^\star$, strategy, nuisance, and random blocks \\
Action-block parameters & $(0.8,0.15)$ for $C^\star$, $(0.6,0.25)$ for strategy, nuisance dim $64$ with scale $4.0$ and noise $0.5$ \\
\framework{} Stage~1 & MLP hidden sizes $(64,64)$, latent dimension $8$, Gumbel temperature $0.5$ \\
Optimization & $50$ epochs, batch size $32$, learning rate $10^{-4}$ \\
Main regularization & $\beta=0.0001$, $\gamma=0.2$ for Pred.+bal. \\
\bottomrule
\end{tabular}
\end{table}

\subsection{Implementation of the experimental comparisons}
\label{app:comparison_implementation}

\paragraph{Code substitutions.}
All configurations use $K=8$ codes. Random codes provide an outcome-independent comparison: synthetic and OfflinePrompts actions receive a fixed random-projection assignment with quantile bins fitted on logged training actions; headlines use seeded balanced random assignments. The $k$-means alternative clusters action features without reward supervision. Predictive-only codes retain the learned architecture and prediction loss but set $\beta=\gamma=0$. Code substitutions retain the RA objective and the same evaluation protocol. For objective substitutions, all methods share the proposed encoder within each seed.

\paragraph{Stage~2 objectives.}
Our Code-IPS objective applies inverse propensity weighting to the observed code, following the marginalization principle of MIPS~\citep{saito2022off}. Code-DR adds a code-level reward prediction and an importance-weighted residual correction, motivated by regression-assisted estimation such as OffCEM~\citep{saito2023off}. Kernel-IS adapts the sentence-level smoothing idea of~\citet{kiyohara2025off}: Gaussian distances from action features to code centroids define normalized soft memberships, and a soft-label propensity model estimates their conditional means. We use the median-distance bandwidth heuristic. The kernel objective weights rewards by policy probabilities times these memberships divided by their estimated conditional means. For all three objectives, inverse propensities are capped at $1/\tau$, where $\tau=0.03$.

We optimize linear softmax code policies on standardized context features and deploy their highest-scoring supported code. RA instead maximizes the fitted code-level outcome directly. IPS and DR use hard-code propensities; kernel-IS uses soft-membership propensities for its objective and support on the text benchmarks. The synthetic comparison uses the common hard-code support at deployment for every objective. Synthetic support is computed from grouped simulator frequencies conditional on the latent context group, following the controlled diagnostic protocol; headline and OfflinePrompts propensities are fitted from logged data.
Within each dataset, the baselines share the candidate generator and deployment evaluator.

\textbf{Main table: configurations and evaluation.}
The synthetic column uses $\alpha_{\mathrm{pred}}=1$, $\alpha_{\mathrm{dep}}=0$, $\kappa_{\mathrm{ov}}=0.6$, 5,000 logged observations, and 10,000 noise-free evaluation samples. Its Stage~1 network has two hidden layers of width~32, learning rate $10^{-3}$, 50 epochs, batch size~32, and $(\beta,\gamma)=(0,0.001)$. We evaluate ideal BCL through grouped simulator conditional rewards. The headline column retains MiniLM features alongside the constructed context and action feature blocks described in Appendix~\ref{app:semisynth_construction}; it uses width~64, learning rate $10^{-4}$, 50 epochs, batch size~32, and $(\beta,\gamma)=(10^{-4},0.2)$. We perform runs across five random seeds. 

\paragraph{Prompt-based reference methods.}
On OfflinePrompts, \emph{prompt-based regression} fits a neural reward predictor to context and logged prompt embeddings, then selects the highest-scoring prompt from the 100-prompt catalog. The \emph{DSO-style prompt policy} learns a context-dependent distribution over that 100-prompt catalog using kernel importance weighting in sentence-embedding space. It uses the logging distribution and an exact sum over the pre-generated prompt outputs to compute the kernel-smoothed logging marginal; importance ratios are clipped at 200. The bandwidth is the median within-movie sentence-embedding distance. Both baselines use the same generated sentences, logged reward data, test contexts, and reward simulator as \framework, but additionally have access to prompt information that is unavailable to \framework. In particular, prompt-based regression observes logged prompt embeddings, while the DSO-style policy accesses the prompt catalog, its generated outputs, and the logging distribution.

\textbf{Compute resources.} Experimental runs were conducted on commodity CPU hardware without GPU acceleration, with typical runs completing in under one hour. The semi-synthetic benchmark additionally required OpenAI API calls for headline generation and validation during dataset construction, which also completed within approximately one hour.
\clearpage

\end{document}